\documentclass[conference]{IEEEtran}

\usepackage{cite}
\usepackage[utf8]{inputenc} % allow utf-8 input
\usepackage[T1]{fontenc}    % use 8-bit T1 fonts
\usepackage{url}            % simple URL typesetting
\usepackage{booktabs}       % professional-quality tables
\usepackage{amsfonts}       % blackboard math symbols
\usepackage{nicefrac}       % compact symbols for 1/2, etc.
\usepackage{microtype}      % microtypography

\usepackage{amsmath,amssymb}
\usepackage[mathscr]{euscript} %To be able to use script letters with \mathscr{}
\usepackage{amsthm} %To be able to write Definitions, Theorems, etc.
\usepackage{tikz} %To draw graphs.
\usepackage{graphicx} % To scale equations
\newcommand*{\Scale}[2][4]{\scalebox{#1}{\ensuremath{#2}}} % To scale equations
\usepackage{framed} % To be able to frame theorems and stuff
\usepackage{dblfloatfix} %To be able to place figures exactly where I want with [H]
\usepackage{multirow}
\usepackage{color}
\usepackage{subfigure}
\usepackage{longtable}
\usepackage[hidelinks, breaklinks=true]{hyperref} %To be able to use links inside the document; [hidelinks removes the ugly red boxes]
\usepackage{xcolor}  \definecolor{shadecolor}{rgb}{.95,.95,.95}  %To put a shaded region
\usepackage[font=footnotesize]{caption}
\usepackage{nicefrac}
\usepackage{dsfont} % To be able to write the indicator function
\usepackage[export]{adjustbox} % To align figures atop using \includegraphics[scale=0.5,valign=t]
\usepackage{cancel} % To cancel/strike things out in mathmode with \cancel{something}.
\usepackage{enumitem}	%To be able to shift itemize with [leftmargin=1.5cm]
\usepackage{wrapfig}
\usepackage[utf8]{inputenc}
\usepackage{arydshln} % To put dotted lines in tables with \begin{tabular}{c:c:c}

\theoremstyle{definition}

\newtheorem{myTheorem}{Theorem}

\tikzstyle{every edge}=  [draw]
\tikzstyle{vertex} = [draw,circle,minimum size=1pt]
\tikzstyle{label} = [minimum size=.1pt,font=\scriptsize]
\tikzstyle{title} = [minimum size=.25cm,font=\small]

\newcommand{\bs}[1]{\boldsymbol{#1}}
\newcommand{\bhat}[1]{\boldsymbol{\hat{#1}}}

\def \R{\mathbb{R}}

\def \1{{\mathds{1}}}
\def \T{\mathsf{T}}

\def \spn{{\rm span}}
\def \Ord{\mathscr{O}}
\def \<{\langle}
\def \>{\rangle}

\DeclareMathOperator*{\argmin}{arg\,min}

\def \Fr{{\hyperref[FrDef]{{\rm F}}}}

\def \K{{\hyperref[KDef]{{\rm K}}}}
\def \n{{\hyperref[nDef]{{\rm n}}}}
\def \d{{\hyperref[dDef]{{\rm d}}}}
\def \r{{\hyperref[rDef]{{\rm r}}}}
\def \xi{{\hyperref[xiDef]{{\rm x}}}}
\def \w{{\hyperref[wDef]{{\rm w}}}}
\def \p{{\hyperref[pDef]{{\rm p}}}}
\def \lambdaa{{\hyperref[lambdaaDef]{\lambda}}}
\def \gammaa{{\hyperref[gammaaDef]{\gamma}}}
\def \kappaa{{\hyperref[kappaaDef]{\kappa}}}
\def \sigmaa{{\hyperref[sigmaaDef]{\sigma}}}
\def \rhoo{{\hyperref[rhooDef]{\rho}}}

\def \x{{\hyperref[xDef]{\bs{{\rm x}}}}}

\def \btheta{{\hyperref[bthetaDef]{\bs{\theta}}}}

\def \I{{\hyperref[IDef]{\bs{{\rm I}}}}}

\def \X{{\hyperref[XDef]{\bs{{\rm X}}}}}
\def \U{{\hyperref[UDef]{\bs{{\rm U}}}}}
\def \P{{\hyperref[PDef]{\bs{{\rm P}}}}}
\def \bnabla{{\hyperref[bnablaDef]{\bs{\nabla}}}}
\def \bTheta{{\hyperref[bThetaDef]{\bs{\Theta}}}}
\def \S{{\hyperref[SDef]{\bs{{\rm S}}}}}
\def \W{{\hyperref[WDef]{\bs{{\rm W}}}}}

\def \i{{\hyperref[iDef]{{\rm i}}}}
\def \j{{\hyperref[jDef]{{\rm j}}}}
\def \k{{\hyperref[kDef]{{\rm k}}}}
\def \t{{\hyperref[tDef]{{\rm t}}}}

\def \O{{\hyperref[ODef]{\Omega}}}
\def \o{{\hyperref[oDef]{\Omega}}}

\def \sU{{\hyperref[sUDef]{\mathbb{U}}}}

\def \FSC{{\hyperref[FSCDef]{{\sc Fsc}}}}

\def \SSC{{\hyperref[SSCDef]{SSC}}}

\def \EM{{\hyperref[EMDef]{EM}}}
\def \UoS{{\hyperref[UoSDef]{UoS}}}
\def \LRMC{{\hyperref[LRMCDef]{LRMC}}}
\def \PCA{{\hyperref[PCADef]{PCA}}}

\def\BibTeX{{\rm B\kern-.05em{\sc i\kern-.025em b}\kern-.08em
    T\kern-.1667em\lower.7ex\hbox{E}\kern-.125emX}}
\begin{document}

%====================TITLE AND AUTHORS

\title{Fusion Subspace Clustering for Incomplete Data
}

% \author{\IEEEauthorblockN{Anonymous Authors}}

\author{\IEEEauthorblockN{Usman Mahmood }
\IEEEauthorblockA{\textit{Department of Computer Science} \\
\textit{Georgia State University }\\
Atlanta, GA, USA \\
umahmood1@gsu.edu}
\and
\IEEEauthorblockN{Daniel Pimentel-Alarcón}
\IEEEauthorblockA{\textit{Department of Biostatistics and Medical Informatics} \\
\textit{University of Wisconsin-Madison}\\
Madison, WI, USA \\
pimentelalar@wisc.edu}
\and

}

\maketitle

\begin{abstract}
This paper introduces {\em fusion subspace clustering}, a novel method to learn low-dimensional structures that approximate large scale yet highly incomplete data. The main idea is to assign each datum to a subspace of its own, and minimize the distance between the subspaces of all data, so that subspaces of the same cluster get {\em fused} together. Our method allows low, high, and even full-rank data; it directly accounts for noise, and its sample complexity approaches the information-theoretic limit. In addition, our approach provides a natural model selection {\em clusterpath}, and a direct completion method. We give convergence guarantees, analyze computational complexity, and show through extensive experiments on real and synthetic data that our approach performs comparably to the state-of-the-art with complete data, and dramatically better if data is missing.

\end{abstract}

%=========================== INTRODUCTION ===================
\section{Introduction}
Inferring low-dimensional structures that explain high-dimensional data has become a cornerstone of discovery in virtually all fields of science. \phantomsection\label{PCADef}Principal component analysis (\PCA), which identifies the low-dimensional linear subspace that best explains a dataset, is arguably the most prominent technique for this purpose. In many applications \---- computer vision, image processing, bioinformatics, linguistics, networks analysis, and more \cite{kanade, charRecog, kanatani, lambertian, recommender, scc, eriksson, guessWho, ssc, network} \---- data is often composed of a mixture of several classes, each of which can be explained by a different subspace. Clustering accordingly is an important unsupervised learning problem that has received tremendous attention in recent years, producing theory and algorithms to handle outliers, noisy measurements, privacy concerns, and data constraints, among other difficulties \cite{scVidal, liu1, liu2, mahdi, qu, peng, wang, aarti, hu, scalableSC, L0sparse, dataDependent, full1, full2, full3, full4, full5, full6, greedySC, LRSC, LSR}.

%Subspace Clustering has numerous contemporary applications, including visualization, compression, consumer segmentation, collaborative filtering, image and video processing, among many others \cite{}.

However, one major contemporary challenge is that data is often incomplete. For example, in image inpainting, the values of some pixels are missing due to faulty sensors and image contamination \cite{inpainting}; in computer vision features are often missing due to occlusions and tracking algorithms malfunctions \cite{occlusions}; in recommender systems each user only rates a limited number of items \cite{collaborativeRanking}; in a network, most nodes communicate in subsets, producing only a handful of all the possible measurements \cite{eriksson}.

{\bf Missing data notoriously complicates clustering.} The main difficulty with highly incomplete data is that subsets of points are rarely observed in overlapping coordinates, which impedes assessing distances. Existing self-expressive formulations \cite{ssc, ewzf}, agglomerative strategies \cite{greedySC}, and partial neighborhoods \cite{hrmc}, all require observing $\Ord(\r+1)$ overlapping coordinates in at least $\K$ sets of $\Ord(\r+1)$ points in order to cluster $\K$ $\r$-dimensional subspaces. In low-sampling regimes, this would require a super-polynomial number of points \cite{hrmc}, which are rarely available in practice. Alternatively, filling missing entries with a sensible value (e.g., zeros or means \cite{ewzf} or using low-rank matrix completion \cite{candes-recht}) may work if data is missing at a rate inversely proportional to the subspaces' dimensions \cite{tsakiris}, or if data is low-rank. However, in most applications data is missing at much higher rates, and due to the number and dimensions of the subspaces, data is typically high or even full-rank. In general, data filled with zeros or means no longer lie in a \phantomsection\label{UoSDef}union of subspaces (\UoS), thus guaranteeing failure even with a modest amount of missing data \cite{elhamifar}. Other approaches include alternating methods like $k$-subspaces \cite{kGROUSE}, expectation-maximization \cite{ssp14}, group-lasso \cite{gssc}, and lifting techniques \cite{elhamifar,greg,ladmc} that require (at the very least) squaring the dimension of an already high-dimensional problem, which severely limits their applicability. More recently methods like \cite{latentFactorAnalysis,missingValueImputation} incorporate a variation of fuzzy c-means for data imputation and/or clustering. However, these existing approaches either have limited applicability or do not perform well if data is missing in large quantities \cite{CompareApproaches}. For example, k-nearest neighbors imputation can distort the data distribution, resulting in inaccurate nearest neighbors identification \cite{KNNProblem}. Regression methods can also lead to low accuracy, especially if the underlying variables have low correlation. Existing approaches along with their weaknesses are compared by \cite{CompareApproaches}. These challenges call the attention to new strategies to address missing data.

%In contrast we measure distances through a proxy.
%
%practical provable algorithm remains elusive.

%\careful{This calls the attention to more holistic (rather than local) approaches to \iSC.}

{\bf This paper} introduces \phantomsection\label{FSCDef}{\em fusion subspace clustering} (\FSC), a novel approach to address incomplete data, inspired by greedy methods, convex relaxations, and fusion penalties \cite{fusionVariable, fusedLasso, groupPursuit, clusterpath, sumofnorms, fusion}. The main idea is to assign each datum to a subspace of its own, and then {\em fuse} together nearby subspaces by minimizing (i) the distance between each datum and its subspace (thus guaranteeing that each datum is explained by its subspace), and (ii) the distance between the subspaces of all data, so that subspaces from points that belong together get fused into one. While \FSC\ is mainly motivated by missing data, it is also new to full-data, and has the next advantages: it allows low, high, and even full-rank data. \FSC\ directly allows noise, and its sample complexity approaches the information-theoretic limit \cite{infoTheoretic}, as shown in sections \ref{sec: Effect of the number of subspaces and data points} and \ref{sec:Effect of missing data}. Similar to hierarchical clustering, \FSC\ can produce a model selection {\em clusterpath} providing detailed information about intra-cluster and cluster-to-cluster distances (see Figure \ref{treeFig}). Finally, its simplicity makes \FSC\ amenable to analysis: our main theoretical result shows that \FSC\ converges to a local minimum. This is particularly remarkable, especially in light that most other subspace clustering algorithms lack theoretical guarantees (even local convergence) when data is missing (except for restrictive fractions of missing entries, and liftings, which are unfeasible for high-dimensional data). Our experiments on real and synthetic data show that with full-data, \FSC\ performs comparably to the state-of-the-art, and dramatically better if data is missing.

\section{Problem Statement}
\label{problemSec}
Let \phantomsection\label{XDef}\phantomsection\label{dDef}\phantomsection\label{nDef}$\X \in \R{}^{\d \times \n}$ be a data matrix whose columns lie {\em approximately} in the union of \phantomsection\label{KDef}$\K$ low-dimensional subspaces of $\R{}^\d$ (i.e., we allow noise).  Assume that we do not know a priori the subspaces, nor how many there are, nor their dimensions, nor which column belongs to which subspace.  Let $\X{}^\O$ denote the incomplete version of $\X$, observed only in the entries of \phantomsection\label{ODef}$\O \subset \{1,\dots,\d\} \times \{1,\dots,\n\}$.  Given $\X{}^\O$, our goals are to cluster the columns of $\X{}^\O$ according to the underlying subspaces, infer such subspaces, and complete $\X{}^\O$.

{\bf Notations.} Throughout the paper, \phantomsection\label{xDef}$\x_\i \in \R{}^\d$ denotes the $\i{}^{\rm th}$ column of $\X$, \phantomsection\label{sUDef}$\sU_\i \subset \R{}^\d$ denotes the subspace assigned to $\x_\i$, and \phantomsection\label{UDef}$\U_\i \in \R{}^{\d \times \r}$ is a basis of $\sU_\i$; here \phantomsection\label{iDef}$\i=1,\dots,\n$, and \phantomsection\label{rDef}$\r$ is an upper bound on the dimension of the subspaces. Given $\i$, we use the superscript \phantomsection\label{oDef}$\o$ to indicate the restriction of a subspace, matrix or vector to the observed entries in $\x_\i$. For example, if $\x_\i$ is observed on $\ell$ rows, then $\x{}_{\i}^\o \in \R{}^{\ell}$ and $\U{}_{\i}^\o \in \R{}^{\ell \times \r}$ denote the restrictions of $\x_\i$ and $\U_\i$ to the observed rows in $\x_\i$. We use \phantomsection\label{FrDef}$\|\bs{\cdot}\|{}_\Fr$ to denote the Frobenius norm, and $\1$ to denote the indicator function.

%This way the burden of missing data is effectively isolated in the first term, while the second term only deals with full-data.
%This effectively isolates the burden of missing data in (i), and measures full-data distances through the proxy subspaces in (ii)
%
%effectively isolates the burden of unobserved entries by , capable of capable of isucceeding in missing data regimes when other algorithms fail
%

%=======================FUSION SUBSPACE CLUSTERING
\section{Fusion Subspace Clustering}
\phantomsection\label{lambdaaDef}First notice that we can write (full-data) subspace clustering as the following optimization problem:

\begin{align}
\label{hierarchicalEq}
\argmin_{\U_1,\dots,\U_\n} \
\sum_{\i=1}^\n \rhoo^2(\x_\i,\U_\i) 
\hspace{.1cm} \text{ s.t. } \hspace{.1cm}
\frac{1}{2} \sum_{\i=1}^\n \sum_{\j=1}^\n \1_{\{\rhoo(\U_\i,\U_\j)=0\}} \leq \lambdaa,
\end{align}
where\phantomsection\label{rhooDef}
\begin{align*}
\rhoo(\x_\i,\U_\i) \ &:= \ \big\| \x_\i-\U_\i (\U_\i^\T \U_\i)^{-1}\U_\i^\T \x_{\i} \big\|_2, \\
\rhoo(\U_\i,\U_\j) \ &:= \ \big\| \U_\i (\U_\i^\T \U_\i)^{-1}\U_\i^\T - \U_\j (\U_\j^\T \U_\j)^{-1}\U_\j^\T \big\|_\Fr.
\end{align*}
Recall that the projector operator onto $\spn\{\U_\i\}$ is $\P_\i:=\U_\i (\U_\i^\T \U_\i)^{-1}\U_\i^\T$, so $\rhoo(\x_\i,\U_\i)$ is simply measuring the distance between a point and a subspace through its projection residual. Similarly, $\rhoo(\U_\i,\U_\j)$ defines a metric on the Grassmannian \cite{grassDist}, measuring distance between subspaces through their projection operators, which unlike bases, are unique for each subspace. Formulation \eqref{hierarchicalEq} assigns each point $\x_\i$ to a subspace of its own (with basis $\U_\i$). Minimizing $\rhoo(\x_\i,\U_\i)$ ensures that each point is close to its assigned subspace. The constraint ensures that there are no more than $\Ord(\sqrt{\lambdaa})$ different subspaces. Notice that if $\lambdaa \geq \n(\n-1)/2$ (the number of distinct pairs of points), then the problem is unconstrained, and a trivial solution is $\U_\i$ formed by $\x_\i$ and any other $\r-1$ vectors (in fact this is precisely our choice for initialization, with the additional $\r-1$ vectors populated with i.i.d.~$\mathscr{N}(0,1)$ entries, known to produce incoherent and nearly orthogonal subspaces with high probability \cite{gaussianIncoherence}). If $\lambdaa=\n(\n-1)/2-1$, then \eqref{hierarchicalEq} forces two subspaces to {\em fuse}, similar to the first step in hierarchical clustering. More generally, if $\lambdaa=\n(\n-1)/2-\ell$, then \eqref{hierarchicalEq} forces $\ell-1$ subspaces to fuse.
%Recall that $\r$ is an upper bound on the subspaces dimensions; Section \ref{modelSelectionSec} includes more details about estimating $\r$ and the dimension of each subspace.
However, \eqref{hierarchicalEq} is a combinatorial problem, so we propose the following relaxation:
\begin{align*}
\argmin_{\U_1,\dots,\U_\n} \
\sum_{\i=1}^\n \rhoo^2(\x_\i,\U_\i) 
\ + \
\frac{\lambdaa}{2} \sum_{\i=1}^\n \sum_{\j=1}^\n \w_{\i\j} \ \rhoo(\U_\i,\U_\j).
\end{align*}
Notice that this is an $\ell_1$-group penalty that promotes sparsity in the terms $\rhoo(\U_\i,\U_\j)$ \cite{antoniadisFan, yuanLin, meier, tibshirani}; however, in our case the groups are not known and need to be discovered, similar to task structure learning \cite{pkdd}. The weights $\w_{\i\j} \geq 0$ quantify how much attention we give to each penalty. Ideally, if $\x_\i$ and $\x_\j$ are not in the same subspace, we would like the penalty $\rhoo(\U_\i,\U_\j)$ to be ignored, so that $\sU_\i$ and $\sU_\j$ do not tilt one another. Conversely, if $\x_\i$ and $\x_\j$ are in the same subspace, we want the penalty $\rhoo(\U_\i,\U_\j)$ to have more weight, so that $\sU_\i$ and $\sU_\j$ get fused together. Here again $\lambdaa \geq 0$ is a proxy of $\K$ that regulates how subspaces fuse together. The larger $\lambdaa$, the more we penalize subspaces being apart, which results in more subspaces getting fused. Sections \ref{modelSelectionSec} and \ref{weightsSec} contain more details about $\lambda$ and $\w_{\i\j}$.

The main reason we find this fusion formulation so attractive is its capacity to effectively isolate the burden of missing data on the first terms, where it can be easily handled. That is because under standard incoherence assumptions, the incomplete-data residual norm $\rhoo(\x_\i^\o,\U_\i^\o):=\|\x_\i^\o-\U_\i^\o (\U_\i^{\o\T} \U_\i^\o)^{-1}\U_\i^{\o\T} \x^\o_\i \|_2$ will be proportionally close to the full-data residual norm $\rhoo(\x_\i,\U_\i)$ with high probability (Theorem 1 in \cite{msdmd}). So if data is missing, all we need to do is replace the first terms, so that each subspace only fits the observed entries of its assigned column. This leaves the second terms unaffected to be used as proxies to compute distances between incomplete points, even if they are observed on disjoint coordinates. With these observations we obtain our general \FSC\ formulation:
\begin{framed}
\begin{align}
\label{ifscEq}
\argmin_{\U_1,\dots,\U_\n} \
\sum_{\i=1}^\n \rhoo^2(\x_\i^\o,\U_\i^\o)
\ + \
\frac{\lambdaa}{2} \sum_{\i=1}^\n \sum_{\j=1}^\n \w_{\i\j} \ \rhoo(\U_\i,\U_\j).
\end{align}
\end{framed}
We point out that \eqref{ifscEq} only requires that each point is {\em close} to its corresponding subspace, as opposed to {\em exactly} on it, so it directly allows noise. To solve \eqref{ifscEq} we use coordinate gradient descent. Since we use $\w_{\i\j} = \sqrt{\r\d} \ \1_\kappaa {\{ \exp(-\gammaa \rhoo^2(\U_\i,\U_\j)}) \}$ (see Section \ref{weightsSec}), the gradient of \eqref{ifscEq} with respect to $\U_\i$ is:
\begin{align*}
\bnabla_\i \ = \ \bnabla_\i' \ &+ \ 4 \lambdaa \sum_{\j \neq \i,  \1_{\kappaa} = 1} 
 \frac{\w_{\i\j} \bnabla''_{\i\j}(1-\gamma \rhoo^2(\U_\i,\U_\j))}{\rhoo(\U_\i,\U_\j)},
\end{align*}
where $\bnabla_\i'$ is equal to $\bs{0}$ for the rows not in $\o$,
\begin{align*}
\bnabla_\i' = &-2 \x_\i^\o \x_\i^{\o\T} \U_\i^\o (\U_\i^{\o\T} \U_\i^\o)^{-1}\\
&+ (\U_\i^\o\U_\i^{\o\T})^2\x_\i^\o\x_\i^{\o\T}\U_\i^\o(\U_\i^{\o\T}\U_\i^\o)^{-1} \\
&+ \U_\i^\o (\U_\i^{\o\T}\U_\i^\o)^{-1} \U_\i^{\o\T}\x_\i^\o\x_\i^{\o\T}\U_\i^\o \U_\i^{\o\T}\U_\i^\o
\end{align*}
for the rows in $\o$, and $\bnabla_{\i\j}'' = (\U_\i(\U_\i^\T\U_\i)^{-1}\U_\i^\T-\I)\U_\j (\U_\j^\T\U_\j)^{-1}\U_\j^\T \U_\i(\U_\i^\T\U_\i)^{-1}$.

Whether data is missing or not, the solution to \eqref{ifscEq} will be a sequence of (full-data) bases $\U_1,\dots,\U_\n$, one for each column in $\X$. Due to the second term in \eqref{ifscEq}, we expect subspace $\sU_\i=\spn\{\U_\i\}$ to be close to $\sU_\j=\spn\{\U_\j\}$ if columns $\x_\i$ and $\x_\j$ belong together, and far otherwise. It remains to group together subspaces that are close, or equivalently, assign a label to each subspace $\sU_\i$. To this end, use spectral clustering, which shows remarkable performance in many modern problems, and is widely used as the final step in many subspace clustering algorithms, including the state-of-the-art \SSC.
Spectral clustering receives a similarity matrix \phantomsection\label{SDef}$\S \in \R{}^{\n \times \n}$ between $\n$ points, and runs a standard clustering method (like $k$-means) on the relevant eigenvectors of the Laplacian matrix of $\S$ \cite{spectral}. We can build a similarity matrix $\S$ between subspaces $\sU_1,\dots,\sU_\n$ whose $(\i,\j){}^{\rm th}$ entry is equal to $1/\rhoo(\U_\i,\U_\j)$. At this point we can run spectral clustering on $\S$ to assign a label \phantomsection\label{kDef}$\k_\i \in \{1,\dots,\K\}$ to each subspace $\sU_\i$, or equivalently, to each column $\x_\i$, thus providing a clustering of $\X$, as desired.

%The entire process of \FSC\ is summarized in Algorithm \ref{fscAlg}.
%
%%=============== ALGORITHM =============
%\begin{algorithm}
%\caption{Fusion Subspace Clustering.}
%\label{fscAlg}
%\Scale[.7]{\gray{1}}. \textbf{Input:} Columns $\x^\o_1,\dots,\x_\n^\o$.
%
%\Scale[.7]{\gray{2}}. Solve \eqref{ifscEq} to obtain $\U_1,\dots,\U_\n$.
%
%\Scale[.7]{\gray{3}}. Compute $\S$, the similarity matrix of $\U_1,\dots,\U_\n$.
%
%\Scale[.7]{\gray{4}}. Spectral cluster $\S$ to obtain labels $\k_1,\k_2,\dots,\k_\n$.
%
%\Scale[.7]{\gray{5}}. \textbf{Output:} The $\i^{\rm th}$ column corresponds to cluster $\k_\i$.
%\end{algorithm}

%============== COMPLETING DATA =================
\subsection{Subspace Estimation and Data Completion}
\label{completionSec}
Recall that our goals are to: (i) cluster the columns of $\X{}^\O$, (ii) infer the underlying subspaces, and (iii) complete $\X{}^\O$. So far we have only achieved (i). However, that is the difficult step. In fact, once $\X{}^\O$ is clustered, there are several straightforward ways to achieve (ii) and (iii).
Common approaches concatenate all the columns of $\X{}^\O$ that correspond to the same cluster into a single matrix $\X{}^\O_{\k}$, and complete it into a matrix $\bhat{\X}_\k$ using low-rank matrix completion (\LRMC) \cite{candes-recht} (because its columns now lie in a single subspace), thus achieving (iii). To accomplish (ii) one can compute the leading singular vectors of $\bhat{\X}_\k$ to produce an subspace basis estimate $\bhat{\U}_\k$. We can do this as well. However, \FSC\ does not require \LRMC, which may fail if the subspaces are coherent (aligned with the canonical axes) or samples are not uniformly spread \cite{candes-recht}.

Our \LRMC-free approach is as follows: since the bases $\U_\i$ produced by \eqref{ifscEq} have no missing data, we can normalize and concatenate all the bases that correspond to the $\k{}^{\rm th}$ cluster into a single matrix \phantomsection\label{WDef}$\W_\k$, and compute its leading singular vectors to produce an {\em average} estimate $\bhat{\U}_\k$, thus achieving (ii). Next we can estimate the coefficient of each $\x{}_\i^\o$ with respect to its corresponding subspace basis $\bhat{\U}_{\k_\i}$: \phantomsection\label{bthetaDef}$\bhat{\btheta}_\i := (\bhat{\U}{}_{\k_\i}^{\o\T} \bhat{\U}{}_{\k_\i}^\o)^{-1}\bhat{\U}{}_{\k_\i}^{\o\T}\x{}_\i^\o$. Since the coefficient of $\x{}_\i^\o$ is the same as the coefficient of $\x_\i$ we can complete $\x{}_\i^\o$ as $\bhat{\x}_\i=\bhat{\U}_{\k_\i}\bhat{\btheta}_\i$, thus achieving (iii).
%\careful{Our LRMC-free method is applied after the subspace bases are identified by \FSC, unlike many existing machine-learning-based methods that work by first filling missing values and then performing clustering.}
In our experiments both \LRMC\ and our approach yield the exact same results up to numerical precision.

\begin{figure*}
\centering
\includegraphics[width=0.95\textwidth]{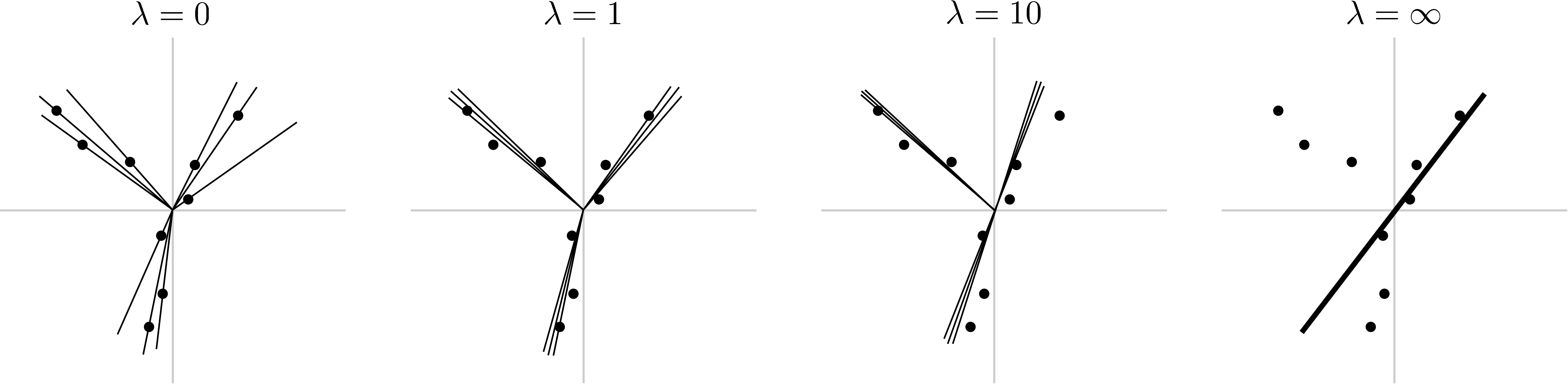}
\caption{
In \eqref{ifscEq}, $\lambdaa \geq 0$ regulates how clusters fuse together. If $\lambdaa=0$, each point is assigned to a subspace that exactly contains it (overfitting). The larger $\lambdaa$, the more we penalize subspaces being apart, which results in subspaces getting closer (as with $\lambdaa=1$), up to the point that some subspaces fuse (as with $\lambdaa=10$). In the extreme ($\lambdaa=\infty$), all subspaces fuse together, and we need to explain all data with a single subspace (which may not be enough). Notice that it is not always evident how many subspaces one should use to explain a dataset. In this illustration, should we choose $\lambdaa=1$, which would result in $3$ subspaces, or $\lambdaa=10$, which would result in $2$? Section \ref{modelSelectionSec} discusses how to choose $\lambdaa$, which in turn determines the number of subspaces $\K$ that best explain the data, and their dimensions.}
\label{fscFig}
\end{figure*}

%=================CONVERGENCE=================
\section{Convergence Guarantees}
Our main theoretical result shows that a sequence of gradient iterates of \FSC\ will converge to a critical point, which is not generally the case for methods of this type \cite{powell, nocedal}.
\begin{myTheorem}
\label{mainThm}
\textit{
The sequence of gradient iterates $\{\U_\i{}^{(\t)}\}_{\i \in [\n], \t>0}$ of  \eqref{ifscEq} has an accumulation point. Moreover, any accumulation point of $\{\U_\i{}^{(\t)}\}_{\i \in [\n], \t>0}$ is a critical point of \eqref{ifscEq}.}
\end{myTheorem}
%Theorem \ref{mainThm} follows by showing that these sequences are bounded, by a use of the Bolzano-Weierstrass theorem, and by using Lemma 3.2 in \cite{beck}. 
\begin{proof}
Since the objective function in \eqref{ifscEq} only depends on $\U_\i$ through its projection matrix, its solution is equivalent to the solution restricting $\U_\i$ to be orthonormal, whence $\|\U_\i\|_\Fr=\r$. It follows that the sequence $\{\U_\i{}^{(\t)}\}_{\i \in [\n], \t>0}$ will be bounded. By the Bolzano Weierstrass theorem, it contains a convergent subsequence, whose limit will be an accumulation point. To show that accumulation points are critical points we will demonstrate that the conditions of Lemma 3.2 in \cite{beck} are met. Let
\begin{align*}
&f(\U_1,\dots,\U_\n) := \frac{\lambdaa}{2} \sum_{\i=1}^\n \sum_{\j=1}^\n \w_{\i\j} \ \rhoo(\U_\i,\U_\j),\\
&g_\i(\U_\i) := \rhoo^2(\x_\i,\U_\i).
\end{align*}
Next notice that:
\begin{enumerate}[label=(\alph*)]
\item
$g_\i$ is closed because it is continuous with a closed domain (see Section A.3.3 in \cite{boyd}). $g_\i$ is also proper convex, because $g_1$ is a norm \cite{ding}. Finally, $g_\i$ is sub differentiable.
\item
$f$ is continuously differentiable, because with $\U_\i$ orthonormal, $\U_\i (\U_\i^\T \U_\i)^{-1}\U_\i^\T$ simplifies to $\U_\i\U_\i{}^\T$, whence $f$ is a polynomial.
\item
The gradient of $f$ is Lipschitz continuous with respect to $\U_\i$, because we are using $\w_{\i\j} = \sqrt{\r\d} \ \1_\kappaa {\{ \exp(-\gammaa \rhoo^2(\U_\i,\U_\j)}) \} \leq \sqrt{\r\d}$ (see Section \ref{weightsSec}), and $\rhoo(\U_\i,\U_\j) \leq \|\U_\i (\U_\i^\T \U_\i)^{-1}\U_\i^\T\|_\Fr + \|\U_\j (\U_\j^\T \U_\j)^{-1}\U_\j^\T\|_\Fr \leq \sqrt{\r} \|\U_\i (\U_\i^\T \U_\i)^{-1}\U_\i^\T\|_2 + \sqrt{\r} \|\U_\j (\U_\j^\T \U_\j)^{-1}\U_\j^\T\|_2 \leq 2\sqrt{\r}$.
\item
The objective function in \eqref{ifscEq} is continuous with closed and bounded domains. It follows by the extreme value theorem that it must have a minimum, and hence a minimizer.
\end{enumerate}
Conditions $(a)$-$(d)$ are the assumptions of Lemma 3.2 in \cite{beck}. Theorem \ref{mainThm} follows directly.
\end{proof}

%============================= MODEL SELECTION =============
\section{Model Selection}
\label{modelSelectionSec}
\FSC\ provides a natural way for model selection, namely determining the number of subspaces $\K$ that best explain the data, and their dimensions.
Intuitively, the first term in \eqref{ifscEq} guarantees that each subspace $\sU_\i$ is close to its assigned column, and the second term guarantees that subspaces from different columns are close to one another. The tradeoff between these two quantities is determined by $\lambdaa \geq 0$ (see Figure \ref{fscFig}). If $\lambdaa=0$, then the second term is ignored, and there is a trivial solution where each subspace exactly contains its assigned column (thus attaining the minimum, zero, in the first term). If $\lambdaa>0$, the second term forces subspaces from different columns to get closer, even if they no longer contain exactly their assigned columns. As $\lambdaa$ grows, subspaces get closer and closer, up to the point where some subspaces {\em fuse} into one. This is verified in our experiments (see Figure \ref{resultsFig}). The extreme case ($\lambdaa=\infty$) forces all subspaces to fuse into one (to attain zero in the second term), meaning we only have one subspace to explain all data, which is precisely \PCA\ (for full-data) and \LRMC\ (for incomplete data). In other words, \FSC\ is a generalization of \PCA\ and \LRMC, which is the particular case of \eqref{ifscEq} with $\lambdaa=\infty$. Iteratively decreasing $\lambdaa$ will result in more and more clusters, until $\lambdaa=0$ produces $\n$. The more subspaces, the more accuracy, but the more degrees of freedom (overfitting). For each $\lambdaa$ that provides a different clustering, we can compute a goodness of fit test (like the Akaike information criterion, AIC \cite{akaike}) that quantifies the tradeoff between accuracy and degrees of freedom, to determine the best number of subspaces $\K$. For example, this test can be in the form of $\K$, and the residuals of the projections of each $\x{}_{\i}^\o$ onto its corresponding $\bhat{\U}{}_\k^\o$, as defined in Section \ref{completionSec}. Similarly, we can iteratively increase $\r$ to find all the columns that lie in $1$-dimensional subspaces, then all the columns that lie in $2$-dimensional subspaces, and so on (pruning the data at each iteration).  This will result in an estimate of the number of subspaces $\K$, and their dimensions.

%================Tree================
\textbf{The Subspace Clusterpath.} Notice that iteratively increasing $\lambdaa$ also provides a natural way to quantify and visualize intra-cluster and outer-cluster similarities through a graph showing the evolution of subspaces as they fuse together, similar to the {\em clusterpath} produced in \cite{clusterpath} for euclidian clustering (Figure \ref{treeFig}). Notice, however, that fusion is not necessarily monotonic, i.e., fused subspaces may split, so in general this graph may be a network, rather than a tree.

%\begin{figure}
%\centering
%\includegraphics[height=2.5cm]{Figures/TreeNk_8.pdf}
%\caption{Subspace clusterpath showing how subspace estimates (corresponding 32 data points in 4 gaussian subspaces, 8 each) progressively fuse as $\lambdaa$ increases.}
%\label{treeFig}
%\end{figure}

\begin{figure*}
\centering
\includegraphics[height=3.45cm]{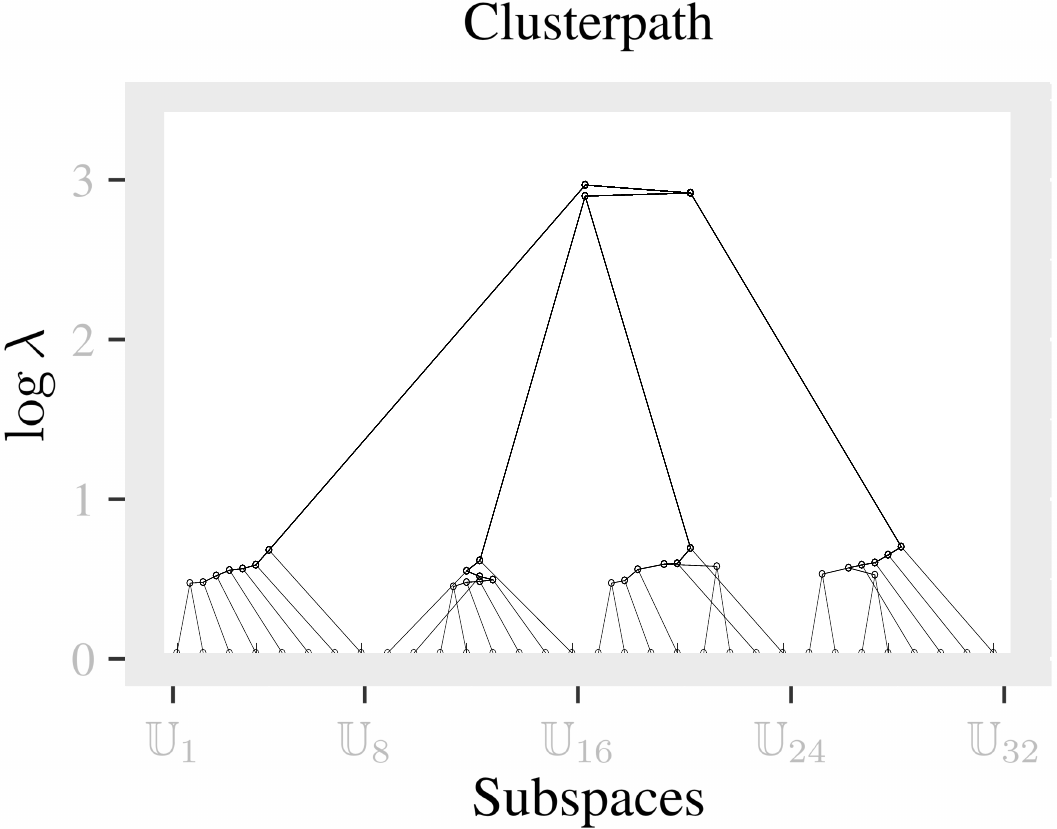} \hspace{.5cm}
\includegraphics[height=3.45cm]{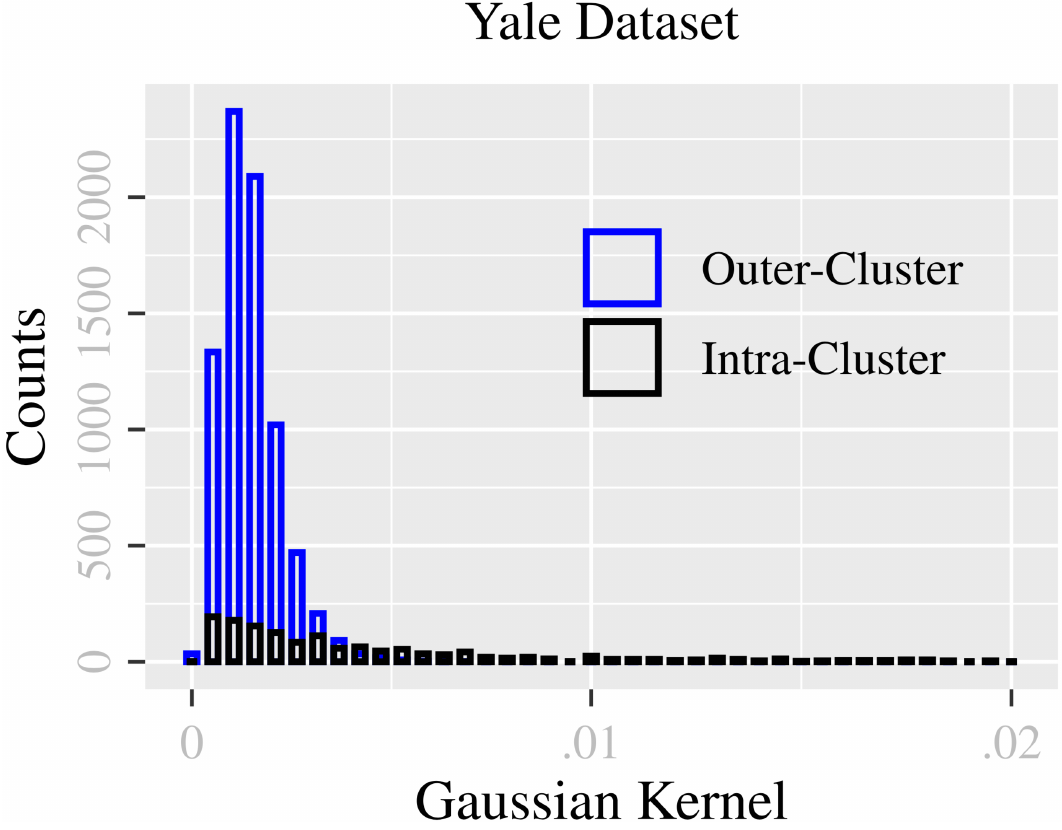} \hspace{.5cm}
\includegraphics[height=3.45cm]{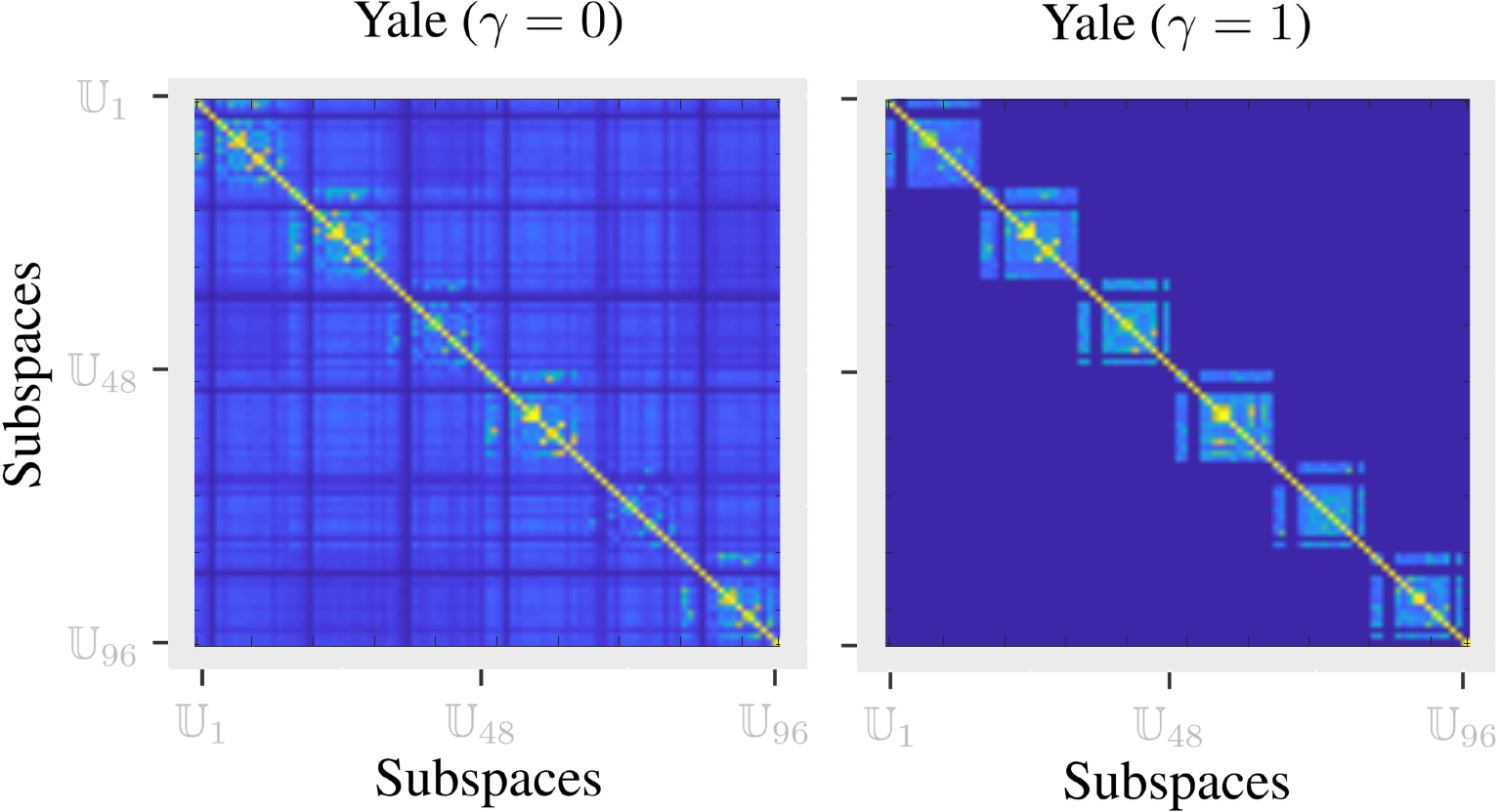}
\caption{\textbf{Left:} Clusterpath showing how subspace estimates (corresponding 32 data points in 4 gaussian subspaces, 8 each) progressively fuse as $\lambdaa$ increases. \textbf{Center:} Distribution of Gaussian kernels in the Yale B dataset with $\gamma=1$; notice that most outer-cluster points receive small values and vice versa, as desired. \textbf{Right:} Similarity matrices produced by \FSC\ on the Yale dataset with $\gamma=0$ (uniform weights, producing a poor clustering because subspaces are not well-separated), and with $\gamma=1$ and $\kappaa=\r$ (Gaussian kernel weights with nearest neighbors, producing a near perfect clustering).}
\label{treeFig}
\end{figure*}

%================WEIGHTS================
\section{Penalty Weights, Computational Complexity, and Parameters}
\label{weightsSec}
Like other fusion formulations \cite{antoniadisFan, yuanLin, meier, tibshirani,pkdd}, \FSC\ involves weight terms $\w_{\i\j}$ that bring the flexibility to distinguish which subspaces to fuse, and which ones not to, which in turn can dramatically improve the clustering quality and computational complexity \cite{chi,tan}. Ideally, we would like $\w_{\i\j}$ to be large if $\x_\i$ and $\x_\j$ lie in the same subspace, so that the penalty $\rhoo(\U_\i,\U_\j)$ gets a higher weight, forcing subspaces $\sU_\i$ and $\sU_\j$ to fuse into one. Conversely, if $\x_\i$ and $\x_\j$ lie in different subspaces, we want $\w_{\i\j}$ to be small, so that the penalty $\rhoo(\U_\i,\U_\j)$ gets ignored, and subspaces $\sU_\i$ and $\sU_\j$ do not fuse.

Here we use \phantomsection\label{kappaaDef}\phantomsection\label{gammaaDef}$\w_{\i\j}= \sqrt{\r\d} \ \1_\kappaa {\{ \exp(-\gammaa \rhoo^2(\U_\i,\U_\j)}) \}$, where the indicator $\1_\kappaa$ takes the value one if $\j$ is amongst the $\kappaa$ nearest neighbors of $\i$ or vice versa, and zero otherwise. Here the factor $\sqrt{\r\d}$ ensures that the penalty is in the order of the number of degrees of freedom (the same rescaling is used in \cite{yuanLin,meier}), and the second factor is a Gaussian kernel that slows the fusion of distant subspaces \cite{chi,tan}, where $\gammaa \geq 0$ regulates how separated subspaces are. In particular $\gammaa=0$ corresponds to uniform weights ($\w_{\i\j}=1$ for every ($\i,\j$)), known to be a good option if subspaces are well-separated, and to produce no {\em splits} in the clusterpath of Euclidean clustering (Theorem 1 in \cite{clusterpath}). More generally, with $\gammaa>0$ the second factor measures the distance between subspaces $\sU_\i$ and $\sU_\j$ such that if $\sU_\i$ and $\sU_\j$ are close, then $\rhoo(\U_\i,\U_\j)$ will be small, resulting in a large value of $\exp(-\gammaa \rhoo^2(\U_\i,\U_\j))$ (and vice versa), as desired. Figure \ref{treeFig} shows the distribution of these Gaussian kernels on the Yale B dataset (which are mostly small for outer-cluster points, and mostly large for intra-cluster points, as desired), together with the similarity matrices produced by two options ($\gamma=0$ and $\gamma=1$).

We point out that besides improving clustering quality, limiting positive weights to nearest neighbors (NNs) improves the computational complexity of \FSC. To see this first notice that finding NNs requires $\n^2$ {\em linear} operations to compute pairwise distances. However, these calculations are negligible in comparison with the {\em polynomial} operations required to compute $\n^2$ gradients (which require matrix inversions). Consequently, by limiting positive weights to $\kappaa$ NNs we cut down these $\n^2$ polynomial calculations to $\kappaa\n$, thus reducing the effective computational complexity of \FSC\ to $\Ord(\d\r\kappaa\n)$ and achieving linear complexity in problem size (see Sections 3.3 and 5.1 in \cite{chi}). This, of course, comes at a price: increasing the effective number of parameters of \FSC\ to a total of four: $\lambdaa$, which controls how much subspaces fuse (see Section \ref{modelSelectionSec}), $\gammaa$ and $\kappaa$, which together determine the weights $\w_{\i\j}$ that control which subspaces fuse, and the gradient step size, which can be tuned with standard techniques; in our experiments we used 5-fold cross-validation.

%======================= EXPERIMENTS ====================
\section{Experiments}
\label{experimentsSec}
Now we study the performance of \FSC. For reference, we compare against the following subspace clustering algorithms that allow missing data:
$(a)$ Entry-wise zero-filling \SSC\ \cite{ewzf}.
$(b)$ \LRMC\ + \SSC\ \cite{ewzf}.
$(c)$ \SSC-Lifting \cite{elhamifar}.
$(d)$ Algebraic variety high-rank matrix completion \cite{greg}.
$(e)$ $k$-subspaces with missing data \cite{kGROUSE}.
$(f)$ \EM\ \cite{ssp14}.
$(g)$ Group-sparse subspace clustering \cite{gssc}.
$(h)$ Mixture subspace clustering \cite{gssc}.
We chose these algorithms based on \cite{gssc, elhamifar}, where they show comparable state-of-the-art performance. $(a)$-$(d)$ essentially run \SSC\ after fill missing entries with zeros, according to a single larger subspace, or after lifting the data. $(e)$-$(h)$ are alternating algorithms that according to \cite{gssc} produce best results when initialized with the output of $(a)$, and so indirectly they also depend on \SSC. To measure performance we compute clustering error (fraction of misclassified points). When applicable (i.e., when no data is missing) we additionally compare against the following full-data approaches: BDR \cite{full2}, iPursuit \cite{full4}, LRR \cite{liu1, liu2}, LSR \cite{LSR}, LRSC \cite{LRSC}, L2Graph \cite{full1}, SCC \cite{scc}, SSC \cite{ssc}, and S3C \cite{full5}.
In the interest of reproducibility, all our code is included in the supplementary material. In the interest of fairness to other algorithms, whenever available we (a) used their code, (b) used their specified parameters, (c) did a sweep to find the best parameters, and (d) used reported results from the literature. Whenever there was a discrepancy, we reported their \textit{best} performance, be that from reports or from our experiments.
%In the interest of reproducibility, all our code is available here \cite{ourCode}.

%\careful{eventhough \FSC\ is for missing data, for reference we compare against full data as well and show that \FSC\ is doing something sensible. However, it is not the best for full-data, but it is for missing data. Consider removing some full-data experiments?}

%as well as:
%\begin{align*}
%\text{Subspace Error} = \sum_{\k=1}^\K \|\P^\star_\k-\bhat{\P}_\k\|_\Fr
%\hspace{1cm} \text{and} \hspace{1cm}
%\text{Completion Error} = \frac{\|\X-\bhat{\X}\|_\Fr}{\|\X\|_\Fr},
%\end{align*}
%where $\P{}^\star_1,\dots,\P{}^\star_\K$ and $\bhat{\P}_1,\dots,\bhat{\P}_\K$ are the projection operators of the true subspaces and their estimates, and $\bhat{\X}$ is the completion of $\X{}^\O$.

%================= SIMULATIONS
\subsection{Simulations}
Since \FSC\ is an entirely new approach to both, full and incomplete data, we present a thorough series of experiments to study its behavior as a function of the penalty parameter $\lambdaa$, the ambient dimension $\d$, the number of subspaces involved $\K$, their dimensions $\r$, the noise variance \phantomsection\label{sigmaDef}$\sigmaa{}^2$, the number of data points in each cluster $\n_\k$, and of course, the fraction of unobserved entries \phantomsection\label{pDef}$\p$. Unless otherwise stated, we use the following default settings: $\d=100$, $\K=4$, $\r=5$, $\sigmaa=0$, $\n_\k=20$, and $\p=0$. We run $30$ trials of each experiment, and show the average results of \FSC\ and all the algorithms above.

\sloppypar In all our simulations we first generate $\K$ matrices $\U{}_\k^\star \in \R{}^{\d \times \r}$ with i.i.d.~$\mathscr{N}(0,1)$ entries, to use as bases of the {\em true} subspaces. For each $\k$ we generate a matrix \phantomsection\label{bThetaDef}$\bTheta{}^\star_\k \in \R{}^{\r \times \n_\k}$, also with i.i.d.~$\mathscr{N}(0,1)$ entries, to use as coefficients of the columns in the $\k{}^{\rm th}$ subspace. We then form $\X$ as the concatenation $[\U_1{}^\star\bTheta{}^\star_1 \ \ \ \ \U_2{}^\star\bTheta{}^\star_2 \ \ \cdots \ \ \U{}_\K^\star\bTheta{}^\star_\K]$, plus a $\d\times\n$ noise matrix with i.i.d.~$\mathscr{N}(0,\sigmaa{}^2)$ entries. To create $\O$, we sample each entry independently with probability $1-\p$.

\begin{figure*}[t]
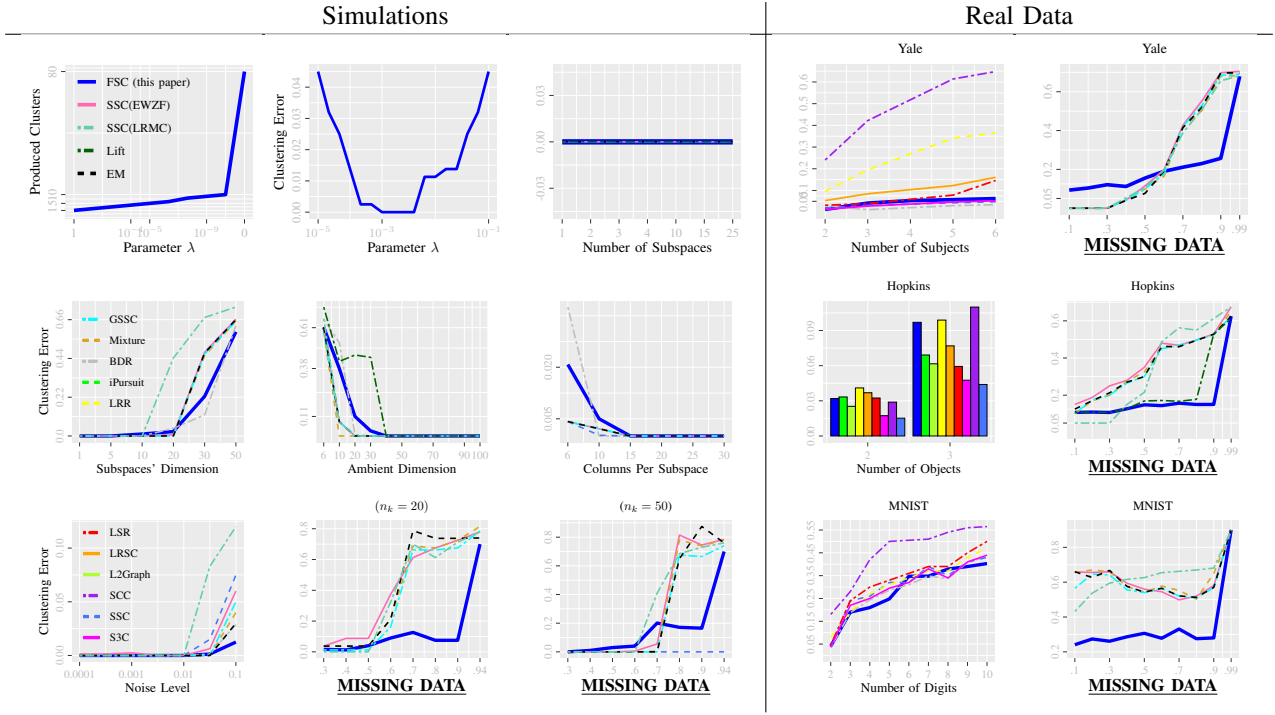

\centering
\begin{tabular}{ccc|cc}
\multicolumn{3}{c|}{Simulations} & \multicolumn{2}{c}{Real Data} \\ \hline
\Scale[.6]{\input{Figures/FiguresSims/LambdaClusters.tex}} & \hspace{-.5cm}
\Scale[.6]{% Created by tikzDevice version 0.12.3 on 2020-05-26 00:05:41
% !TEX encoding = UTF-8 Unicode
\begin{tikzpicture}[x=1pt,y=1pt]
\definecolor{fillColor}{RGB}{255,255,255}
\path[use as bounding box,fill=fillColor,fill opacity=0.00] (0,0) rectangle (151.77,144.54);
\begin{scope}
\path[clip] (  0.00,  0.00) rectangle (151.77,144.54);
\definecolor{drawColor}{RGB}{255,255,255}
\definecolor{fillColor}{RGB}{255,255,255}

\path[draw=drawColor,line width= 0.6pt,line join=round,line cap=round,fill=fillColor] (  0.00,  0.00) rectangle (151.77,144.54);
\end{scope}
\begin{scope}
\path[clip] ( 28.02, 28.18) rectangle (146.27,125.77);
\definecolor{fillColor}{gray}{0.92}

\path[fill=fillColor] ( 28.02, 28.18) rectangle (146.27,125.77);
\definecolor{drawColor}{RGB}{255,255,255}

\path[draw=drawColor,line width= 0.3pt,line join=round] ( 28.02, 42.47) --
	(146.27, 42.47);

\path[draw=drawColor,line width= 0.3pt,line join=round] ( 28.02, 62.19) --
	(146.27, 62.19);

\path[draw=drawColor,line width= 0.3pt,line join=round] ( 28.02, 81.90) --
	(146.27, 81.90);

\path[draw=drawColor,line width= 0.3pt,line join=round] ( 28.02,101.62) --
	(146.27,101.62);

\path[draw=drawColor,line width= 0.3pt,line join=round] ( 28.02,121.33) --
	(146.27,121.33);

\path[draw=drawColor,line width= 0.3pt,line join=round] ( 53.55, 28.18) --
	( 53.55,125.77);

\path[draw=drawColor,line width= 0.3pt,line join=round] (107.30, 28.18) --
	(107.30,125.77);

\path[draw=drawColor,line width= 0.6pt,line join=round] ( 28.02, 32.62) --
	(146.27, 32.62);

\path[draw=drawColor,line width= 0.6pt,line join=round] ( 28.02, 52.33) --
	(146.27, 52.33);

\path[draw=drawColor,line width= 0.6pt,line join=round] ( 28.02, 72.05) --
	(146.27, 72.05);

\path[draw=drawColor,line width= 0.6pt,line join=round] ( 28.02, 91.76) --
	(146.27, 91.76);

\path[draw=drawColor,line width= 0.6pt,line join=round] ( 28.02,111.48) --
	(146.27,111.48);

\path[draw=drawColor,line width= 0.6pt,line join=round] ( 33.40, 28.18) --
	( 33.40,125.77);

\path[draw=drawColor,line width= 0.6pt,line join=round] ( 73.71, 28.18) --
	( 73.71,125.77);

\path[draw=drawColor,line width= 0.6pt,line join=round] (140.89, 28.18) --
	(140.89,125.77);
\definecolor{drawColor}{RGB}{0,0,255}

\path[draw=drawColor,line width= 1.7pt,line join=round] ( 33.40,121.33) --
	( 40.12, 95.70) --
	( 46.84, 81.90) --
	( 53.55, 59.82) --
	( 60.27, 37.54) --
	( 66.99, 37.54) --
	( 73.71, 32.62) --
	( 80.43, 32.62) --
	( 87.15, 32.62) --
	( 93.86, 32.62) --
	(100.58, 54.89) --
	(107.30, 54.89) --
	(114.02, 59.82) --
	(120.74, 59.82) --
	(127.46, 81.90) --
	(134.17, 95.70) --
	(140.89,121.33);
\end{scope}
\begin{scope}
\path[clip] (  0.00,  0.00) rectangle (151.77,144.54);
\definecolor{drawColor}{RGB}{190,190,190}

\node[text=drawColor,rotate= 90.00,anchor=base,inner sep=0pt, outer sep=0pt, scale=  0.77] at ( 22.33, 32.62) {0.00};

\node[text=drawColor,rotate= 90.00,anchor=base,inner sep=0pt, outer sep=0pt, scale=  0.77] at ( 22.33, 52.33) {0.01};

\node[text=drawColor,rotate= 90.00,anchor=base,inner sep=0pt, outer sep=0pt, scale=  0.77] at ( 22.33, 72.05) {0.02};

\node[text=drawColor,rotate= 90.00,anchor=base,inner sep=0pt, outer sep=0pt, scale=  0.77] at ( 22.33, 91.76) {0.03};

\node[text=drawColor,rotate= 90.00,anchor=base,inner sep=0pt, outer sep=0pt, scale=  0.77] at ( 22.33,111.48) {0.04};
\end{scope}
\begin{scope}
\path[clip] (  0.00,  0.00) rectangle (151.77,144.54);
\definecolor{drawColor}{gray}{0.20}

\path[draw=drawColor,line width= 0.6pt,line join=round] ( 25.27, 32.62) --
	( 28.02, 32.62);

\path[draw=drawColor,line width= 0.6pt,line join=round] ( 25.27, 52.33) --
	( 28.02, 52.33);

\path[draw=drawColor,line width= 0.6pt,line join=round] ( 25.27, 72.05) --
	( 28.02, 72.05);

\path[draw=drawColor,line width= 0.6pt,line join=round] ( 25.27, 91.76) --
	( 28.02, 91.76);

\path[draw=drawColor,line width= 0.6pt,line join=round] ( 25.27,111.48) --
	( 28.02,111.48);
\end{scope}
\begin{scope}
\path[clip] (  0.00,  0.00) rectangle (151.77,144.54);
\definecolor{drawColor}{gray}{0.20}

\path[draw=drawColor,line width= 0.6pt,line join=round] ( 33.40, 25.43) --
	( 33.40, 28.18);

\path[draw=drawColor,line width= 0.6pt,line join=round] ( 73.71, 25.43) --
	( 73.71, 28.18);

\path[draw=drawColor,line width= 0.6pt,line join=round] (140.89, 25.43) --
	(140.89, 28.18);
\end{scope}
\begin{scope}
\path[clip] (  0.00,  0.00) rectangle (151.77,144.54);
\definecolor{drawColor}{RGB}{190,190,190}

\node[text=drawColor,anchor=base,inner sep=0pt, outer sep=0pt, scale=  0.78] at ( 33.40, 17.07) {$10^{-5}$};

\node[text=drawColor,anchor=base,inner sep=0pt, outer sep=0pt, scale=  0.78] at ( 73.71, 17.07) {$10^{-3}$};

\node[text=drawColor,anchor=base,inner sep=0pt, outer sep=0pt, scale=  0.78] at (140.89, 17.07) {$10^{-1}$};
\end{scope}
\begin{scope}
\path[clip] (  0.00,  0.00) rectangle (151.77,144.54);
\definecolor{drawColor}{RGB}{0,0,0}

\node[text=drawColor,anchor=base,inner sep=0pt, outer sep=0pt, scale=  0.91] at ( 87.15,  7.27) {Parameter $\lambda$};
\end{scope}
\begin{scope}
\path[clip] (  0.00,  0.00) rectangle (151.77,144.54);
\definecolor{drawColor}{RGB}{0,0,0}

\node[text=drawColor,rotate= 90.00,anchor=base,inner sep=0pt, outer sep=0pt, scale=  0.91] at ( 11.79, 76.97) {Clustering Error};
\end{scope}
\end{tikzpicture}} & \hspace{-.5cm}
\Scale[.6]{% Created by tikzDevice version 0.12.3 on 2020-05-26 00:05:33
% !TEX encoding = UTF-8 Unicode
\begin{tikzpicture}[x=1pt,y=1pt]
\definecolor{fillColor}{RGB}{255,255,255}
\path[use as bounding box,fill=fillColor,fill opacity=0.00] (0,0) rectangle (151.77,144.54);
\begin{scope}
\path[clip] (  0.00,  0.00) rectangle (151.77,144.54);
\definecolor{drawColor}{RGB}{255,255,255}
\definecolor{fillColor}{RGB}{255,255,255}

\path[draw=drawColor,line width= 0.6pt,line join=round,line cap=round,fill=fillColor] (  0.00,  0.00) rectangle (151.77,144.54);
\end{scope}
\begin{scope}
\path[clip] ( 28.02, 28.18) rectangle (146.27,125.77);
\definecolor{fillColor}{gray}{0.92}

\path[fill=fillColor] ( 28.02, 28.18) rectangle (146.27,125.77);
\definecolor{drawColor}{RGB}{255,255,255}

\path[draw=drawColor,line width= 0.3pt,line join=round] ( 28.02, 33.06) --
	(146.27, 33.06);

\path[draw=drawColor,line width= 0.3pt,line join=round] ( 28.02, 62.34) --
	(146.27, 62.34);

\path[draw=drawColor,line width= 0.3pt,line join=round] ( 28.02, 91.61) --
	(146.27, 91.61);

\path[draw=drawColor,line width= 0.3pt,line join=round] ( 28.02,120.89) --
	(146.27,120.89);

\path[draw=drawColor,line width= 0.3pt,line join=round] ( 42.36, 28.18) --
	( 42.36,125.77);

\path[draw=drawColor,line width= 0.3pt,line join=round] ( 60.27, 28.18) --
	( 60.27,125.77);

\path[draw=drawColor,line width= 0.3pt,line join=round] ( 78.19, 28.18) --
	( 78.19,125.77);

\path[draw=drawColor,line width= 0.3pt,line join=round] ( 96.10, 28.18) --
	( 96.10,125.77);

\path[draw=drawColor,line width= 0.3pt,line join=round] (114.02, 28.18) --
	(114.02,125.77);

\path[draw=drawColor,line width= 0.3pt,line join=round] (131.93, 28.18) --
	(131.93,125.77);

\path[draw=drawColor,line width= 0.6pt,line join=round] ( 28.02, 47.70) --
	(146.27, 47.70);

\path[draw=drawColor,line width= 0.6pt,line join=round] ( 28.02, 76.97) --
	(146.27, 76.97);

\path[draw=drawColor,line width= 0.6pt,line join=round] ( 28.02,106.25) --
	(146.27,106.25);

\path[draw=drawColor,line width= 0.6pt,line join=round] ( 33.40, 28.18) --
	( 33.40,125.77);

\path[draw=drawColor,line width= 0.6pt,line join=round] ( 51.31, 28.18) --
	( 51.31,125.77);

\path[draw=drawColor,line width= 0.6pt,line join=round] ( 69.23, 28.18) --
	( 69.23,125.77);

\path[draw=drawColor,line width= 0.6pt,line join=round] ( 87.15, 28.18) --
	( 87.15,125.77);

\path[draw=drawColor,line width= 0.6pt,line join=round] (105.06, 28.18) --
	(105.06,125.77);

\path[draw=drawColor,line width= 0.6pt,line join=round] (122.98, 28.18) --
	(122.98,125.77);

\path[draw=drawColor,line width= 0.6pt,line join=round] (140.89, 28.18) --
	(140.89,125.77);
\definecolor{drawColor}{RGB}{0,0,255}

\path[draw=drawColor,line width= 3.4pt,line join=round] ( 33.40, 76.97) --
	( 51.31, 76.97) --
	( 69.23, 76.97) --
	( 87.15, 76.97) --
	(105.06, 76.97) --
	(122.98, 76.97) --
	(140.89, 76.97);
\definecolor{drawColor}{RGB}{190,190,190}

\path[draw=drawColor,line width= 1.1pt,dash pattern=on 2pt off 2pt on 6pt off 2pt ,line join=round] ( 33.40, 76.97) --
	( 51.31, 76.97) --
	( 69.23, 76.97) --
	( 87.15, 76.97) --
	(105.06, 76.97) --
	(122.98, 76.97) --
	(140.89, 76.97);

\path[] ( 33.40, 76.97) --
	( 51.31, 76.97) --
	( 69.23, 76.97) --
	( 87.15, 76.97) --
	(105.06, 76.97) --
	(122.98, 76.97) --
	(140.89, 76.97);
\definecolor{drawColor}{RGB}{72,118,255}

\path[draw=drawColor,line width= 1.1pt,dash pattern=on 4pt off 4pt ,line join=round] ( 33.40, 76.97) --
	( 51.31, 76.97) --
	( 69.23, 76.97) --
	( 87.15, 76.97) --
	(105.06, 76.97) --
	(122.98, 76.97) --
	(140.89, 76.97);

\path[] ( 33.40, 76.97) --
	( 51.31, 76.97) --
	( 69.23, 76.97) --
	( 87.15, 76.97) --
	(105.06, 76.97) --
	(122.98, 76.97) --
	(140.89, 76.97);

\path[] ( 33.40, 76.97) --
	( 51.31, 76.97) --
	( 69.23, 76.97) --
	( 87.15, 76.97) --
	(105.06, 76.97) --
	(122.98, 76.97) --
	(140.89, 76.97);

\path[] ( 33.40, 76.97) --
	( 51.31, 76.97) --
	( 69.23, 76.97) --
	( 87.15, 76.97) --
	(105.06, 76.97) --
	(122.98, 76.97) --
	(140.89, 76.97);

\path[] ( 33.40, 76.97) --
	( 51.31, 76.97) --
	( 69.23, 76.97) --
	( 87.15, 76.97) --
	(105.06, 76.97) --
	(122.98, 76.97) --
	(140.89, 76.97);
\definecolor{drawColor}{RGB}{255,105,180}

\path[draw=drawColor,line width= 1.1pt,line join=round] ( 33.40, 76.97) --
	( 51.31, 76.97) --
	( 69.23, 76.97) --
	( 87.15, 76.97) --
	(105.06, 76.97) --
	(122.98, 76.97) --
	(140.89, 76.97);
\definecolor{drawColor}{RGB}{218,165,32}

\path[draw=drawColor,line width= 1.1pt,dash pattern=on 4pt off 4pt ,line join=round] ( 33.40, 76.97) --
	( 51.31, 76.97) --
	( 69.23, 76.97) --
	( 87.15, 76.97) --
	(105.06, 76.97) --
	(122.98, 76.97) --
	(140.89, 76.97);
\definecolor{drawColor}{RGB}{0,255,255}

\path[draw=drawColor,line width= 1.1pt,dash pattern=on 2pt off 2pt on 6pt off 2pt ,line join=round] ( 33.40, 76.97) --
	( 51.31, 76.97) --
	( 69.23, 76.97) --
	( 87.15, 76.97) --
	(105.06, 76.97) --
	(122.98, 76.97) --
	(140.89, 76.97);
\definecolor{drawColor}{RGB}{0,0,0}

\path[draw=drawColor,line width= 1.1pt,dash pattern=on 4pt off 4pt ,line join=round] ( 33.40, 76.97) --
	( 51.31, 76.97) --
	( 69.23, 76.97) --
	( 87.15, 76.97) --
	(105.06, 76.97) --
	(122.98, 76.97) --
	(140.89, 76.97);
\definecolor{drawColor}{RGB}{102,205,170}

\path[draw=drawColor,line width= 1.1pt,dash pattern=on 2pt off 2pt on 6pt off 2pt ,line join=round] ( 33.40, 76.97) --
	( 51.31, 76.97) --
	( 69.23, 76.97) --
	( 87.15, 76.97) --
	(105.06, 76.97) --
	(122.98, 76.97) --
	(140.89, 76.97);
\definecolor{drawColor}{RGB}{0,100,0}

\path[draw=drawColor,line width= 1.1pt,dash pattern=on 2pt off 2pt on 6pt off 2pt ,line join=round] ( 33.40, 76.97) --
	( 51.31, 76.97) --
	( 69.23, 76.97) --
	( 87.15, 76.97) --
	(105.06, 76.97) --
	(122.98, 76.97) --
	(140.89, 76.97);

\path[] ( 33.40, 76.97) --
	( 51.31, 76.97) --
	( 69.23, 76.97) --
	( 87.15, 76.97) --
	(105.06, 76.97) --
	(122.98, 76.97) --
	(140.89, 76.97);

\path[] ( 33.40, 76.97) --
	( 51.31, 76.97) --
	( 69.23, 76.97) --
	( 87.15, 76.97) --
	(105.06, 76.97) --
	(122.98, 76.97) --
	(140.89, 76.97);
\end{scope}
\begin{scope}
\path[clip] (  0.00,  0.00) rectangle (151.77,144.54);
\definecolor{drawColor}{RGB}{190,190,190}

\node[text=drawColor,rotate= 90.00,anchor=base,inner sep=0pt, outer sep=0pt, scale=  0.77] at ( 22.33, 47.70) {-0.03};

\node[text=drawColor,rotate= 90.00,anchor=base,inner sep=0pt, outer sep=0pt, scale=  0.77] at ( 22.33, 76.97) {0.00};

\node[text=drawColor,rotate= 90.00,anchor=base,inner sep=0pt, outer sep=0pt, scale=  0.77] at ( 22.33,106.25) {0.03};
\end{scope}
\begin{scope}
\path[clip] (  0.00,  0.00) rectangle (151.77,144.54);
\definecolor{drawColor}{gray}{0.20}

\path[draw=drawColor,line width= 0.6pt,line join=round] ( 25.27, 47.70) --
	( 28.02, 47.70);

\path[draw=drawColor,line width= 0.6pt,line join=round] ( 25.27, 76.97) --
	( 28.02, 76.97);

\path[draw=drawColor,line width= 0.6pt,line join=round] ( 25.27,106.25) --
	( 28.02,106.25);
\end{scope}
\begin{scope}
\path[clip] (  0.00,  0.00) rectangle (151.77,144.54);
\definecolor{drawColor}{gray}{0.20}

\path[draw=drawColor,line width= 0.6pt,line join=round] ( 33.40, 25.43) --
	( 33.40, 28.18);

\path[draw=drawColor,line width= 0.6pt,line join=round] ( 51.31, 25.43) --
	( 51.31, 28.18);

\path[draw=drawColor,line width= 0.6pt,line join=round] ( 69.23, 25.43) --
	( 69.23, 28.18);

\path[draw=drawColor,line width= 0.6pt,line join=round] ( 87.15, 25.43) --
	( 87.15, 28.18);

\path[draw=drawColor,line width= 0.6pt,line join=round] (105.06, 25.43) --
	(105.06, 28.18);

\path[draw=drawColor,line width= 0.6pt,line join=round] (122.98, 25.43) --
	(122.98, 28.18);

\path[draw=drawColor,line width= 0.6pt,line join=round] (140.89, 25.43) --
	(140.89, 28.18);
\end{scope}
\begin{scope}
\path[clip] (  0.00,  0.00) rectangle (151.77,144.54);
\definecolor{drawColor}{RGB}{190,190,190}

\node[text=drawColor,anchor=base,inner sep=0pt, outer sep=0pt, scale=  0.78] at ( 33.40, 17.07) {1};

\node[text=drawColor,anchor=base,inner sep=0pt, outer sep=0pt, scale=  0.78] at ( 51.31, 17.07) {2};

\node[text=drawColor,anchor=base,inner sep=0pt, outer sep=0pt, scale=  0.78] at ( 69.23, 17.07) {3};

\node[text=drawColor,anchor=base,inner sep=0pt, outer sep=0pt, scale=  0.78] at ( 87.15, 17.07) {4};

\node[text=drawColor,anchor=base,inner sep=0pt, outer sep=0pt, scale=  0.78] at (105.06, 17.07) {10};

\node[text=drawColor,anchor=base,inner sep=0pt, outer sep=0pt, scale=  0.78] at (122.98, 17.07) {15};

\node[text=drawColor,anchor=base,inner sep=0pt, outer sep=0pt, scale=  0.78] at (140.89, 17.07) {25};
\end{scope}
\begin{scope}
\path[clip] (  0.00,  0.00) rectangle (151.77,144.54);
\definecolor{drawColor}{RGB}{0,0,0}

\node[text=drawColor,anchor=base,inner sep=0pt, outer sep=0pt, scale=  0.91] at ( 87.15,  7.27) {Number of Subspaces};
\end{scope}
\end{tikzpicture}} & \hspace{-.25cm}
\Scale[.6]{% Created by tikzDevice version 0.12.3 on 2020-05-26 00:07:31
% !TEX encoding = UTF-8 Unicode
\begin{tikzpicture}[x=1pt,y=1pt]
\definecolor{fillColor}{RGB}{255,255,255}
\path[use as bounding box,fill=fillColor,fill opacity=0.00] (0,0) rectangle (151.77,144.54);
\begin{scope}
\path[clip] (  0.00,  0.00) rectangle (151.77,144.54);
\definecolor{drawColor}{RGB}{255,255,255}
\definecolor{fillColor}{RGB}{255,255,255}

\path[draw=drawColor,line width= 0.6pt,line join=round,line cap=round,fill=fillColor] (  0.00,  0.00) rectangle (151.77,144.54);
\end{scope}
\begin{scope}
\path[clip] ( 28.02, 28.18) rectangle (146.27,125.77);
\definecolor{fillColor}{gray}{0.92}

\path[fill=fillColor] ( 28.02, 28.18) rectangle (146.27,125.77);
\definecolor{drawColor}{RGB}{255,255,255}

\path[draw=drawColor,line width= 0.3pt,line join=round] ( 28.02,121.61) --
	(146.27,121.61);

\path[draw=drawColor,line width= 0.3pt,line join=round] ( 28.02,107.92) --
	(146.27,107.92);

\path[draw=drawColor,line width= 0.3pt,line join=round] ( 28.02, 94.22) --
	(146.27, 94.22);

\path[draw=drawColor,line width= 0.3pt,line join=round] ( 28.02, 80.53) --
	(146.27, 80.53);

\path[draw=drawColor,line width= 0.3pt,line join=round] ( 28.02, 66.84) --
	(146.27, 66.84);

\path[draw=drawColor,line width= 0.3pt,line join=round] ( 28.02, 53.15) --
	(146.27, 53.15);

\path[draw=drawColor,line width= 0.3pt,line join=round] ( 28.02, 42.88) --
	(146.27, 42.88);

\path[draw=drawColor,line width= 0.3pt,line join=round] ( 46.84, 28.18) --
	( 46.84,125.77);

\path[draw=drawColor,line width= 0.3pt,line join=round] ( 73.71, 28.18) --
	( 73.71,125.77);

\path[draw=drawColor,line width= 0.3pt,line join=round] (100.58, 28.18) --
	(100.58,125.77);

\path[draw=drawColor,line width= 0.3pt,line join=round] (127.46, 28.18) --
	(127.46,125.77);

\path[draw=drawColor,line width= 0.6pt,line join=round] ( 28.02,114.76) --
	(146.27,114.76);

\path[draw=drawColor,line width= 0.6pt,line join=round] ( 28.02,101.07) --
	(146.27,101.07);

\path[draw=drawColor,line width= 0.6pt,line join=round] ( 28.02, 87.38) --
	(146.27, 87.38);

\path[draw=drawColor,line width= 0.6pt,line join=round] ( 28.02, 73.69) --
	(146.27, 73.69);

\path[draw=drawColor,line width= 0.6pt,line join=round] ( 28.02, 60.00) --
	(146.27, 60.00);

\path[draw=drawColor,line width= 0.6pt,line join=round] ( 28.02, 46.31) --
	(146.27, 46.31);

\path[draw=drawColor,line width= 0.6pt,line join=round] ( 28.02, 39.46) --
	(146.27, 39.46);

\path[draw=drawColor,line width= 0.6pt,line join=round] ( 33.40, 28.18) --
	( 33.40,125.77);

\path[draw=drawColor,line width= 0.6pt,line join=round] ( 60.27, 28.18) --
	( 60.27,125.77);

\path[draw=drawColor,line width= 0.6pt,line join=round] ( 87.15, 28.18) --
	( 87.15,125.77);

\path[draw=drawColor,line width= 0.6pt,line join=round] (114.02, 28.18) --
	(114.02,125.77);

\path[draw=drawColor,line width= 0.6pt,line join=round] (140.89, 28.18) --
	(140.89,125.77);
\definecolor{drawColor}{RGB}{0,0,255}

\path[draw=drawColor,line width= 2.3pt,line join=round] ( 33.40, 34.38) --
	( 60.27, 38.34) --
	( 87.15, 39.89) --
	(114.02, 40.49) --
	(140.89, 41.17);
\definecolor{drawColor}{RGB}{190,190,190}

\path[draw=drawColor,line width= 1.1pt,dash pattern=on 2pt off 2pt on 6pt off 2pt ,line join=round] ( 33.40, 36.68) --
	( 60.27, 34.19) --
	( 87.15, 35.46) --
	(114.02, 36.72) --
	(140.89, 37.31);
\definecolor{drawColor}{RGB}{160,32,240}

\path[draw=drawColor,line width= 1.1pt,dash pattern=on 2pt off 2pt on 6pt off 2pt ,line join=round] ( 33.40, 65.50) --
	( 60.27, 90.38) --
	( 87.15,103.50) --
	(114.02,116.62) --
	(140.89,121.33);
\definecolor{drawColor}{RGB}{72,118,255}

\path[draw=drawColor,line width= 1.1pt,dash pattern=on 4pt off 4pt ,line join=round] ( 33.40, 35.16) --
	( 60.27, 36.86) --
	( 87.15, 37.69) --
	(114.02, 38.52) --
	(140.89, 39.21);
\definecolor{drawColor}{RGB}{255,255,0}

\path[draw=drawColor,line width= 1.1pt,dash pattern=on 4pt off 4pt ,line join=round] ( 33.40, 45.65) --
	( 60.27, 59.34) --
	( 87.15, 69.36) --
	(114.02, 79.38) --
	(140.89, 82.57);
\definecolor{drawColor}{RGB}{255,0,0}

\path[draw=drawColor,line width= 1.1pt,dash pattern=on 2pt off 2pt on 6pt off 2pt ,line join=round] ( 33.40, 36.94) --
	( 60.27, 38.04) --
	( 87.15, 40.70) --
	(114.02, 43.36) --
	(140.89, 52.62);
\definecolor{drawColor}{RGB}{255,0,255}

\path[draw=drawColor,line width= 1.1pt,line join=round] ( 33.40, 34.38) --
	( 60.27, 36.44) --
	( 87.15, 37.72) --
	(114.02, 39.00) --
	(140.89, 39.78);
\definecolor{drawColor}{RGB}{255,165,0}

\path[draw=drawColor,line width= 1.1pt,line join=round] ( 33.40, 39.90) --
	( 60.27, 44.21) --
	( 87.15, 46.79) --
	(114.02, 49.37) --
	(140.89, 54.60);

\path[] ( 33.40, 32.62) --
	( 60.27, 32.62) --
	( 87.15, 32.62) --
	(114.02, 32.62) --
	(140.89, 32.62);

\path[] ( 33.40, 32.62) --
	( 60.27, 32.62) --
	( 87.15, 32.62) --
	(114.02, 32.62) --
	(140.89, 32.62);

\path[] ( 33.40, 32.62) --
	( 60.27, 32.62) --
	( 87.15, 32.62) --
	(114.02, 32.62) --
	(140.89, 32.62);

\path[] ( 33.40, 32.62) --
	( 60.27, 32.62) --
	( 87.15, 32.62) --
	(114.02, 32.62) --
	(140.89, 32.62);

\path[] ( 33.40, 32.62) --
	( 60.27, 32.62) --
	( 87.15, 32.62) --
	(114.02, 32.62) --
	(140.89, 32.62);

\path[] ( 33.40, 32.62) --
	( 60.27, 32.62) --
	( 87.15, 32.62) --
	(114.02, 32.62) --
	(140.89, 32.62);

\path[] ( 33.40, 32.62) --
	( 60.27, 32.62) --
	( 87.15, 32.62) --
	(114.02, 32.62) --
	(140.89, 32.62);

\path[] ( 33.40, 32.62) --
	( 60.27, 32.62) --
	( 87.15, 32.62) --
	(114.02, 32.62) --
	(140.89, 32.62);
\end{scope}
\begin{scope}
\path[clip] (  0.00,  0.00) rectangle (151.77,144.54);
\definecolor{drawColor}{RGB}{190,190,190}

\node[text=drawColor,rotate= 90.00,anchor=base,inner sep=0pt, outer sep=0pt, scale=  0.77] at ( 22.33,114.76) {0.6};

\node[text=drawColor,rotate= 90.00,anchor=base,inner sep=0pt, outer sep=0pt, scale=  0.77] at ( 22.33,101.07) {0.5};

\node[text=drawColor,rotate= 90.00,anchor=base,inner sep=0pt, outer sep=0pt, scale=  0.77] at ( 22.33, 87.38) {0.4};

\node[text=drawColor,rotate= 90.00,anchor=base,inner sep=0pt, outer sep=0pt, scale=  0.77] at ( 22.33, 73.69) {0.3};

\node[text=drawColor,rotate= 90.00,anchor=base,inner sep=0pt, outer sep=0pt, scale=  0.77] at ( 22.33, 60.00) {0.2};

\node[text=drawColor,rotate= 90.00,anchor=base,inner sep=0pt, outer sep=0pt, scale=  0.77] at ( 22.33, 46.31) {0.1};

\node[text=drawColor,rotate= 90.00,anchor=base,inner sep=0pt, outer sep=0pt, scale=  0.77] at ( 22.33, 39.46) {0.05};
\end{scope}
\begin{scope}
\path[clip] (  0.00,  0.00) rectangle (151.77,144.54);
\definecolor{drawColor}{gray}{0.20}

\path[draw=drawColor,line width= 0.6pt,line join=round] ( 25.27,114.76) --
	( 28.02,114.76);

\path[draw=drawColor,line width= 0.6pt,line join=round] ( 25.27,101.07) --
	( 28.02,101.07);

\path[draw=drawColor,line width= 0.6pt,line join=round] ( 25.27, 87.38) --
	( 28.02, 87.38);

\path[draw=drawColor,line width= 0.6pt,line join=round] ( 25.27, 73.69) --
	( 28.02, 73.69);

\path[draw=drawColor,line width= 0.6pt,line join=round] ( 25.27, 60.00) --
	( 28.02, 60.00);

\path[draw=drawColor,line width= 0.6pt,line join=round] ( 25.27, 46.31) --
	( 28.02, 46.31);

\path[draw=drawColor,line width= 0.6pt,line join=round] ( 25.27, 39.46) --
	( 28.02, 39.46);
\end{scope}
\begin{scope}
\path[clip] (  0.00,  0.00) rectangle (151.77,144.54);
\definecolor{drawColor}{gray}{0.20}

\path[draw=drawColor,line width= 0.6pt,line join=round] ( 33.40, 25.43) --
	( 33.40, 28.18);

\path[draw=drawColor,line width= 0.6pt,line join=round] ( 60.27, 25.43) --
	( 60.27, 28.18);

\path[draw=drawColor,line width= 0.6pt,line join=round] ( 87.15, 25.43) --
	( 87.15, 28.18);

\path[draw=drawColor,line width= 0.6pt,line join=round] (114.02, 25.43) --
	(114.02, 28.18);

\path[draw=drawColor,line width= 0.6pt,line join=round] (140.89, 25.43) --
	(140.89, 28.18);
\end{scope}
\begin{scope}
\path[clip] (  0.00,  0.00) rectangle (151.77,144.54);
\definecolor{drawColor}{RGB}{190,190,190}

\node[text=drawColor,anchor=base,inner sep=0pt, outer sep=0pt, scale=  0.78] at ( 33.40, 17.07) {2};

\node[text=drawColor,anchor=base,inner sep=0pt, outer sep=0pt, scale=  0.78] at ( 60.27, 17.07) {3};

\node[text=drawColor,anchor=base,inner sep=0pt, outer sep=0pt, scale=  0.78] at ( 87.15, 17.07) {4};

\node[text=drawColor,anchor=base,inner sep=0pt, outer sep=0pt, scale=  0.78] at (114.02, 17.07) {5};

\node[text=drawColor,anchor=base,inner sep=0pt, outer sep=0pt, scale=  0.78] at (140.89, 17.07) {6};
\end{scope}
\begin{scope}
\path[clip] (  0.00,  0.00) rectangle (151.77,144.54);
\definecolor{drawColor}{RGB}{0,0,0}

\node[text=drawColor,anchor=base,inner sep=0pt, outer sep=0pt, scale=  0.91] at ( 87.15,  7.27) {Number of Subjects};
\end{scope}
\begin{scope}
\path[clip] (  0.00,  0.00) rectangle (151.77,144.54);
\definecolor{drawColor}{RGB}{0,0,0}

\node[text=drawColor,anchor=base,inner sep=0pt, outer sep=0pt, scale=  0.88] at ( 87.15,132.98) {Yale};
\end{scope}
\end{tikzpicture}} & \hspace{-.5cm}
\Scale[.6]{% Created by tikzDevice version 0.12.3 on 2020-05-26 15:07:50
% !TEX encoding = UTF-8 Unicode
\begin{tikzpicture}[x=1pt,y=1pt]
\definecolor{fillColor}{RGB}{255,255,255}
\path[use as bounding box,fill=fillColor,fill opacity=0.00] (0,0) rectangle (151.77,144.54);
\begin{scope}
\path[clip] (  0.00,  0.00) rectangle (151.77,144.54);
\definecolor{drawColor}{RGB}{255,255,255}
\definecolor{fillColor}{RGB}{255,255,255}

\path[draw=drawColor,line width= 0.6pt,line join=round,line cap=round,fill=fillColor] (  0.00, -0.00) rectangle (151.77,144.54);
\end{scope}
\begin{scope}
\path[clip] ( 28.02, 30.80) rectangle (146.27,125.77);
\definecolor{fillColor}{gray}{0.92}

\path[fill=fillColor] ( 28.02, 30.80) rectangle (146.27,125.77);
\definecolor{drawColor}{RGB}{255,255,255}

\path[draw=drawColor,line width= 0.3pt,line join=round] ( 28.02,120.84) --
	(146.27,120.84);

\path[draw=drawColor,line width= 0.3pt,line join=round] ( 28.02, 96.35) --
	(146.27, 96.35);

\path[draw=drawColor,line width= 0.3pt,line join=round] ( 28.02, 71.86) --
	(146.27, 71.86);

\path[draw=drawColor,line width= 0.3pt,line join=round] ( 28.02, 50.43) --
	(146.27, 50.43);

\path[draw=drawColor,line width= 0.3pt,line join=round] ( 45.34, 30.80) --
	( 45.34,125.77);

\path[draw=drawColor,line width= 0.3pt,line join=round] ( 69.23, 30.80) --
	( 69.23,125.77);

\path[draw=drawColor,line width= 0.3pt,line join=round] ( 93.12, 30.80) --
	( 93.12,125.77);

\path[draw=drawColor,line width= 0.3pt,line join=round] (117.00, 30.80) --
	(117.00,125.77);

\path[draw=drawColor,line width= 0.3pt,line join=round] (134.92, 30.80) --
	(134.92,125.77);

\path[draw=drawColor,line width= 0.6pt,line join=round] ( 28.02,108.59) --
	(146.27,108.59);

\path[draw=drawColor,line width= 0.6pt,line join=round] ( 28.02, 84.10) --
	(146.27, 84.10);

\path[draw=drawColor,line width= 0.6pt,line join=round] ( 28.02, 59.61) --
	(146.27, 59.61);

\path[draw=drawColor,line width= 0.6pt,line join=round] ( 28.02, 41.24) --
	(146.27, 41.24);

\path[draw=drawColor,line width= 0.6pt,line join=round] ( 33.40, 30.80) --
	( 33.40,125.77);

\path[draw=drawColor,line width= 0.6pt,line join=round] ( 57.29, 30.80) --
	( 57.29,125.77);

\path[draw=drawColor,line width= 0.6pt,line join=round] ( 81.17, 30.80) --
	( 81.17,125.77);

\path[draw=drawColor,line width= 0.6pt,line join=round] (105.06, 30.80) --
	(105.06,125.77);

\path[draw=drawColor,line width= 0.6pt,line join=round] (128.95, 30.80) --
	(128.95,125.77);

\path[draw=drawColor,line width= 0.6pt,line join=round] (140.89, 30.80) --
	(140.89,125.77);
\definecolor{drawColor}{RGB}{0,0,255}

\path[draw=drawColor,line width= 2.3pt,line join=round] ( 33.40, 46.53) --
	( 45.34, 47.85) --
	( 57.29, 49.94) --
	( 69.23, 48.83) --
	( 81.17, 54.10) --
	( 93.12, 58.40) --
	(105.06, 60.98) --
	(117.00, 63.28) --
	(128.95, 66.59) --
	(140.89,118.39);

\path[] ( 33.40, 35.12) --
	( 45.34, 35.12) --
	( 57.29, 35.12) --
	( 69.23, 35.12) --
	( 81.17, 35.12) --
	( 93.12, 35.12) --
	(105.06, 35.12) --
	(117.00, 35.12) --
	(128.95, 35.12) --
	(140.89, 35.12);

\path[] ( 33.40, 35.12) --
	( 45.34, 35.12) --
	( 57.29, 35.12) --
	( 69.23, 35.12) --
	( 81.17, 35.12) --
	( 93.12, 35.12) --
	(105.06, 35.12) --
	(117.00, 35.12) --
	(128.95, 35.12) --
	(140.89, 35.12);

\path[] ( 33.40, 35.12) --
	( 45.34, 35.12) --
	( 57.29, 35.12) --
	( 69.23, 35.12) --
	( 81.17, 35.12) --
	( 93.12, 35.12) --
	(105.06, 35.12) --
	(117.00, 35.12) --
	(128.95, 35.12) --
	(140.89, 35.12);

\path[] ( 33.40, 35.12) --
	( 45.34, 35.12) --
	( 57.29, 35.12) --
	( 69.23, 35.12) --
	( 81.17, 35.12) --
	( 93.12, 35.12) --
	(105.06, 35.12) --
	(117.00, 35.12) --
	(128.95, 35.12) --
	(140.89, 35.12);

\path[] ( 33.40, 35.12) --
	( 45.34, 35.12) --
	( 57.29, 35.12) --
	( 69.23, 35.12) --
	( 81.17, 35.12) --
	( 93.12, 35.12) --
	(105.06, 35.12) --
	(117.00, 35.12) --
	(128.95, 35.12) --
	(140.89, 35.12);

\path[] ( 33.40, 35.12) --
	( 45.34, 35.12) --
	( 57.29, 35.12) --
	( 69.23, 35.12) --
	( 81.17, 35.12) --
	( 93.12, 35.12) --
	(105.06, 35.12) --
	(117.00, 35.12) --
	(128.95, 35.12) --
	(140.89, 35.12);

\path[] ( 33.40, 35.12) --
	( 45.34, 35.12) --
	( 57.29, 35.12) --
	( 69.23, 35.12) --
	( 81.17, 35.12) --
	( 93.12, 35.12) --
	(105.06, 35.12) --
	(117.00, 35.12) --
	(128.95, 35.12) --
	(140.89, 35.12);
\definecolor{drawColor}{RGB}{255,105,180}

\path[draw=drawColor,line width= 1.1pt,line join=round] ( 33.40, 35.12) --
	( 45.34, 35.12) --
	( 57.29, 35.12) --
	( 69.23, 40.23) --
	( 81.17, 49.15) --
	( 93.12, 59.35) --
	(105.06, 87.42) --
	(117.00,102.73) --
	(128.95,120.58) --
	(140.89,121.45);
\definecolor{drawColor}{RGB}{218,165,32}

\path[draw=drawColor,line width= 1.1pt,dash pattern=on 4pt off 4pt ,line join=round] ( 33.40, 35.12) --
	( 45.34, 35.12) --
	( 57.29, 35.12) --
	( 69.23, 40.17) --
	( 81.17, 47.24) --
	( 93.12, 55.20) --
	(105.06, 86.18) --
	(117.00, 96.72) --
	(128.95,118.39) --
	(140.89,118.88);
\definecolor{drawColor}{RGB}{0,255,255}

\path[draw=drawColor,line width= 1.1pt,dash pattern=on 2pt off 2pt on 6pt off 2pt ,line join=round] ( 33.40, 35.12) --
	( 45.34, 35.12) --
	( 57.29, 35.12) --
	( 69.23, 40.23) --
	( 81.17, 47.30) --
	( 93.12, 56.79) --
	(105.06, 86.80) --
	(117.00, 99.90) --
	(128.95,119.00) --
	(140.89,119.74);
\definecolor{drawColor}{RGB}{0,0,0}

\path[draw=drawColor,line width= 1.1pt,dash pattern=on 4pt off 4pt ,line join=round] ( 33.40, 35.12) --
	( 45.34, 35.12) --
	( 57.29, 35.12) --
	( 69.23, 40.02) --
	( 81.17, 44.46) --
	( 93.12, 58.51) --
	(105.06, 85.94) --
	(117.00, 98.67) --
	(128.95,120.47) --
	(140.89,120.47);
\definecolor{drawColor}{RGB}{102,205,170}

\path[draw=drawColor,line width= 1.1pt,dash pattern=on 2pt off 2pt on 6pt off 2pt ,line join=round] ( 33.40, 35.12) --
	( 45.34, 35.12) --
	( 57.29, 35.12) --
	( 69.23, 41.61) --
	( 81.17, 48.11) --
	( 93.12, 59.00) --
	(105.06, 82.63) --
	(117.00, 97.33) --
	(128.95,115.57) --
	(140.89,118.64);

\path[] ( 33.40, 35.12) --
	( 45.34, 35.12) --
	( 57.29, 35.12) --
	( 69.23, 35.12) --
	( 81.17, 35.12) --
	( 93.12, 35.12) --
	(105.06, 35.12) --
	(117.00, 35.12) --
	(128.95, 35.12) --
	(140.89, 35.12);

\path[] ( 33.40, 35.12) --
	( 45.34, 35.12) --
	( 57.29, 35.12) --
	( 69.23, 35.12) --
	( 81.17, 35.12) --
	( 93.12, 35.12) --
	(105.06, 35.12) --
	(117.00, 35.12) --
	(128.95, 35.12) --
	(140.89, 35.12);

\path[] ( 33.40, 35.12) --
	( 45.34, 35.12) --
	( 57.29, 35.12) --
	( 69.23, 35.12) --
	( 81.17, 35.12) --
	( 93.12, 35.12) --
	(105.06, 35.12) --
	(117.00, 35.12) --
	(128.95, 35.12) --
	(140.89, 35.12);
\end{scope}
\begin{scope}
\path[clip] (  0.00,  0.00) rectangle (151.77,144.54);
\definecolor{drawColor}{RGB}{190,190,190}

\node[text=drawColor,rotate= 90.00,anchor=base,inner sep=0pt, outer sep=0pt, scale=  0.77] at ( 22.33,108.59) {0.6};

\node[text=drawColor,rotate= 90.00,anchor=base,inner sep=0pt, outer sep=0pt, scale=  0.77] at ( 22.33, 84.10) {0.4};

\node[text=drawColor,rotate= 90.00,anchor=base,inner sep=0pt, outer sep=0pt, scale=  0.77] at ( 22.33, 59.61) {0.2};

\node[text=drawColor,rotate= 90.00,anchor=base,inner sep=0pt, outer sep=0pt, scale=  0.77] at ( 22.33, 41.24) {0.05};
\end{scope}
\begin{scope}
\path[clip] (  0.00,  0.00) rectangle (151.77,144.54);
\definecolor{drawColor}{gray}{0.20}

\path[draw=drawColor,line width= 0.6pt,line join=round] ( 25.27,108.59) --
	( 28.02,108.59);

\path[draw=drawColor,line width= 0.6pt,line join=round] ( 25.27, 84.10) --
	( 28.02, 84.10);

\path[draw=drawColor,line width= 0.6pt,line join=round] ( 25.27, 59.61) --
	( 28.02, 59.61);

\path[draw=drawColor,line width= 0.6pt,line join=round] ( 25.27, 41.24) --
	( 28.02, 41.24);
\end{scope}
\begin{scope}
\path[clip] (  0.00,  0.00) rectangle (151.77,144.54);
\definecolor{drawColor}{gray}{0.20}

\path[draw=drawColor,line width= 0.6pt,line join=round] ( 33.40, 28.05) --
	( 33.40, 30.80);

\path[draw=drawColor,line width= 0.6pt,line join=round] ( 57.29, 28.05) --
	( 57.29, 30.80);

\path[draw=drawColor,line width= 0.6pt,line join=round] ( 81.17, 28.05) --
	( 81.17, 30.80);

\path[draw=drawColor,line width= 0.6pt,line join=round] (105.06, 28.05) --
	(105.06, 30.80);

\path[draw=drawColor,line width= 0.6pt,line join=round] (128.95, 28.05) --
	(128.95, 30.80);

\path[draw=drawColor,line width= 0.6pt,line join=round] (140.89, 28.05) --
	(140.89, 30.80);
\end{scope}
\begin{scope}
\path[clip] (  0.00,  0.00) rectangle (151.77,144.54);
\definecolor{drawColor}{RGB}{190,190,190}

\node[text=drawColor,anchor=base,inner sep=0pt, outer sep=0pt, scale=  0.78] at ( 33.40, 19.70) {.1};

\node[text=drawColor,anchor=base,inner sep=0pt, outer sep=0pt, scale=  0.78] at ( 57.29, 19.70) {.3};

\node[text=drawColor,anchor=base,inner sep=0pt, outer sep=0pt, scale=  0.78] at ( 81.17, 19.70) {.5};

\node[text=drawColor,anchor=base,inner sep=0pt, outer sep=0pt, scale=  0.78] at (105.06, 19.70) {.7};

\node[text=drawColor,anchor=base,inner sep=0pt, outer sep=0pt, scale=  0.78] at (128.95, 19.70) {.9};

\node[text=drawColor,anchor=base,inner sep=0pt, outer sep=0pt, scale=  0.78] at (140.89, 19.70) {.99};
\end{scope}
\begin{scope}
\path[clip] (  0.00,  0.00) rectangle (151.77,144.54);
\definecolor{drawColor}{RGB}{0,0,0}

\node[text=drawColor,anchor=base,inner sep=0pt, outer sep=0pt, scale=  1.21] at ( 87.15,  7.85) {\underline{\textbf{MISSING DATA}}};
\end{scope}
\begin{scope}
\path[clip] (  0.00,  0.00) rectangle (151.77,144.54);
\definecolor{drawColor}{RGB}{0,0,0}

\node[text=drawColor,anchor=base,inner sep=0pt, outer sep=0pt, scale=  0.88] at ( 87.15,132.98) {Yale};
\end{scope}
\end{tikzpicture}}
\\
\Scale[.55]{\input{Figures/FiguresSims/SubDimError.tex}} & \hspace{-.5cm}
\Scale[.55]{\input{Figures/FiguresSims/AmbientDimError.tex}} & \hspace{-.5cm}
\Scale[.55]{% Created by tikzDevice version 0.12.3 on 2020-05-26 16:13:33
% !TEX encoding = UTF-8 Unicode
\begin{tikzpicture}[x=1pt,y=1pt]
\definecolor{fillColor}{RGB}{255,255,255}
\path[use as bounding box,fill=fillColor,fill opacity=0.00] (0,0) rectangle (151.77,144.54);
\begin{scope}
\path[clip] (  0.00,  0.00) rectangle (151.77,144.54);
\definecolor{drawColor}{RGB}{255,255,255}
\definecolor{fillColor}{RGB}{255,255,255}

\path[draw=drawColor,line width= 0.6pt,line join=round,line cap=round,fill=fillColor] (  0.00,  0.00) rectangle (151.77,144.54);
\end{scope}
\begin{scope}
\path[clip] ( 28.02, 28.18) rectangle (146.27,125.77);
\definecolor{fillColor}{gray}{0.92}

\path[fill=fillColor] ( 28.02, 28.18) rectangle (146.27,125.77);
\definecolor{drawColor}{RGB}{255,255,255}

\path[draw=drawColor,line width= 0.3pt,line join=round] ( 28.02,115.42) --
	(146.27,115.42);

\path[draw=drawColor,line width= 0.3pt,line join=round] ( 28.02, 97.67) --
	(146.27, 97.67);

\path[draw=drawColor,line width= 0.3pt,line join=round] ( 28.02, 62.19) --
	(146.27, 62.19);

\path[draw=drawColor,line width= 0.3pt,line join=round] ( 44.15, 28.18) --
	( 44.15,125.77);

\path[draw=drawColor,line width= 0.3pt,line join=round] ( 65.65, 28.18) --
	( 65.65,125.77);

\path[draw=drawColor,line width= 0.3pt,line join=round] ( 87.15, 28.18) --
	( 87.15,125.77);

\path[draw=drawColor,line width= 0.3pt,line join=round] (108.64, 28.18) --
	(108.64,125.77);

\path[draw=drawColor,line width= 0.3pt,line join=round] (130.14, 28.18) --
	(130.14,125.77);

\path[draw=drawColor,line width= 0.6pt,line join=round] ( 28.02, 79.93) --
	(146.27, 79.93);

\path[draw=drawColor,line width= 0.6pt,line join=round] ( 28.02, 44.44) --
	(146.27, 44.44);

\path[draw=drawColor,line width= 0.6pt,line join=round] ( 33.40, 28.18) --
	( 33.40,125.77);

\path[draw=drawColor,line width= 0.6pt,line join=round] ( 54.90, 28.18) --
	( 54.90,125.77);

\path[draw=drawColor,line width= 0.6pt,line join=round] ( 76.40, 28.18) --
	( 76.40,125.77);

\path[draw=drawColor,line width= 0.6pt,line join=round] ( 97.89, 28.18) --
	( 97.89,125.77);

\path[draw=drawColor,line width= 0.6pt,line join=round] (119.39, 28.18) --
	(119.39,125.77);

\path[draw=drawColor,line width= 0.6pt,line join=round] (140.89, 28.18) --
	(140.89,125.77);
\definecolor{drawColor}{RGB}{0,0,255}

\path[draw=drawColor,line width= 2.3pt,line join=round] ( 33.40, 81.82) --
	( 54.90, 44.44) --
	( 76.40, 32.62) --
	( 97.89, 32.62) --
	(119.39, 32.62) --
	(140.89, 32.62);
\definecolor{drawColor}{RGB}{190,190,190}

\path[draw=drawColor,line width= 1.1pt,dash pattern=on 2pt off 2pt on 6pt off 2pt ,line join=round] ( 33.40,121.33) --
	( 54.90, 32.62) --
	( 76.40, 32.62) --
	( 97.89, 32.62) --
	(119.39, 32.62) --
	(140.89, 32.62);

\path[] ( 33.40, 32.62) --
	( 54.90, 32.62) --
	( 76.40, 32.62) --
	( 97.89, 32.62) --
	(119.39, 32.62) --
	(140.89, 32.62);
\definecolor{drawColor}{RGB}{72,118,255}

\path[draw=drawColor,line width= 1.1pt,dash pattern=on 4pt off 4pt ,line join=round] ( 33.40, 42.55) --
	( 54.90, 33.11) --
	( 76.40, 32.62) --
	( 97.89, 32.62) --
	(119.39, 32.62) --
	(140.89, 32.62);

\path[] ( 33.40, 32.62) --
	( 54.90, 32.62) --
	( 76.40, 32.62) --
	( 97.89, 32.62) --
	(119.39, 32.62) --
	(140.89, 32.62);

\path[] ( 33.40, 32.62) --
	( 54.90, 32.62) --
	( 76.40, 32.62) --
	( 97.89, 32.62) --
	(119.39, 32.62) --
	(140.89, 32.62);

\path[] ( 33.40, 32.62) --
	( 54.90, 32.62) --
	( 76.40, 32.62) --
	( 97.89, 32.62) --
	(119.39, 32.62) --
	(140.89, 32.62);

\path[] ( 33.40, 32.62) --
	( 54.90, 32.62) --
	( 76.40, 32.62) --
	( 97.89, 32.62) --
	(119.39, 32.62) --
	(140.89, 32.62);
\definecolor{drawColor}{RGB}{255,105,180}

\path[draw=drawColor,line width= 1.1pt,line join=round] ( 33.40, 42.55) --
	( 54.90, 37.58) --
	( 76.40, 32.62) --
	( 97.89, 32.62) --
	(119.39, 32.62) --
	(140.89, 32.62);
\definecolor{drawColor}{RGB}{218,165,32}

\path[draw=drawColor,line width= 1.1pt,dash pattern=on 4pt off 4pt ,line join=round] ( 33.40, 42.55) --
	( 54.90, 37.58) --
	( 76.40, 32.62) --
	( 97.89, 32.62) --
	(119.39, 32.62) --
	(140.89, 32.62);
\definecolor{drawColor}{RGB}{0,255,255}

\path[draw=drawColor,line width= 1.1pt,dash pattern=on 2pt off 2pt on 6pt off 2pt ,line join=round] ( 33.40, 42.55) --
	( 54.90, 37.58) --
	( 76.40, 32.62) --
	( 97.89, 32.62) --
	(119.39, 32.62) --
	(140.89, 32.62);
\definecolor{drawColor}{RGB}{0,0,0}

\path[draw=drawColor,line width= 1.1pt,dash pattern=on 4pt off 4pt ,line join=round] ( 33.40, 42.55) --
	( 54.90, 37.58) --
	( 76.40, 32.62) --
	( 97.89, 32.62) --
	(119.39, 32.62) --
	(140.89, 32.62);

\path[] ( 33.40, 32.62) --
	( 54.90, 32.62) --
	( 76.40, 32.62) --
	( 97.89, 32.62) --
	(119.39, 32.62) --
	(140.89, 32.62);

\path[] ( 33.40, 32.62) --
	( 54.90, 32.62) --
	( 76.40, 32.62) --
	( 97.89, 32.62) --
	(119.39, 32.62) --
	(140.89, 32.62);

\path[] ( 33.40, 32.62) --
	( 54.90, 32.62) --
	( 76.40, 32.62) --
	( 97.89, 32.62) --
	(119.39, 32.62) --
	(140.89, 32.62);

\path[] ( 33.40, 32.62) --
	( 54.90, 32.62) --
	( 76.40, 32.62) --
	( 97.89, 32.62) --
	(119.39, 32.62) --
	(140.89, 32.62);
\end{scope}
\begin{scope}
\path[clip] (  0.00,  0.00) rectangle (151.77,144.54);
\definecolor{drawColor}{RGB}{190,190,190}

\node[text=drawColor,rotate= 90.00,anchor=base,inner sep=0pt, outer sep=0pt, scale=  0.77] at ( 22.33, 79.93) {0.020};

\node[text=drawColor,rotate= 90.00,anchor=base,inner sep=0pt, outer sep=0pt, scale=  0.77] at ( 22.33, 44.44) {0.005};
\end{scope}
\begin{scope}
\path[clip] (  0.00,  0.00) rectangle (151.77,144.54);
\definecolor{drawColor}{gray}{0.20}

\path[draw=drawColor,line width= 0.6pt,line join=round] ( 25.27, 79.93) --
	( 28.02, 79.93);

\path[draw=drawColor,line width= 0.6pt,line join=round] ( 25.27, 44.44) --
	( 28.02, 44.44);
\end{scope}
\begin{scope}
\path[clip] (  0.00,  0.00) rectangle (151.77,144.54);
\definecolor{drawColor}{gray}{0.20}

\path[draw=drawColor,line width= 0.6pt,line join=round] ( 33.40, 25.43) --
	( 33.40, 28.18);

\path[draw=drawColor,line width= 0.6pt,line join=round] ( 54.90, 25.43) --
	( 54.90, 28.18);

\path[draw=drawColor,line width= 0.6pt,line join=round] ( 76.40, 25.43) --
	( 76.40, 28.18);

\path[draw=drawColor,line width= 0.6pt,line join=round] ( 97.89, 25.43) --
	( 97.89, 28.18);

\path[draw=drawColor,line width= 0.6pt,line join=round] (119.39, 25.43) --
	(119.39, 28.18);

\path[draw=drawColor,line width= 0.6pt,line join=round] (140.89, 25.43) --
	(140.89, 28.18);
\end{scope}
\begin{scope}
\path[clip] (  0.00,  0.00) rectangle (151.77,144.54);
\definecolor{drawColor}{RGB}{190,190,190}

\node[text=drawColor,anchor=base,inner sep=0pt, outer sep=0pt, scale=  0.78] at ( 33.40, 17.07) {6};

\node[text=drawColor,anchor=base,inner sep=0pt, outer sep=0pt, scale=  0.78] at ( 54.90, 17.07) {10};

\node[text=drawColor,anchor=base,inner sep=0pt, outer sep=0pt, scale=  0.78] at ( 76.40, 17.07) {15};

\node[text=drawColor,anchor=base,inner sep=0pt, outer sep=0pt, scale=  0.78] at ( 97.89, 17.07) {20};

\node[text=drawColor,anchor=base,inner sep=0pt, outer sep=0pt, scale=  0.78] at (119.39, 17.07) {25};

\node[text=drawColor,anchor=base,inner sep=0pt, outer sep=0pt, scale=  0.78] at (140.89, 17.07) {30};
\end{scope}
\begin{scope}
\path[clip] (  0.00,  0.00) rectangle (151.77,144.54);
\definecolor{drawColor}{RGB}{0,0,0}

\node[text=drawColor,anchor=base,inner sep=0pt, outer sep=0pt, scale=  0.91] at ( 87.15,  7.27) {Columns Per Subspace};
\end{scope}
\end{tikzpicture}} & \hspace{-.25cm}
\Scale[.55]{% Created by tikzDevice version 0.12.3 on 2020-05-26 00:07:53
% !TEX encoding = UTF-8 Unicode
\begin{tikzpicture}[x=1pt,y=1pt]
\definecolor{fillColor}{RGB}{255,255,255}
\path[use as bounding box,fill=fillColor,fill opacity=0.00] (0,0) rectangle (151.77,144.54);
\begin{scope}
\path[clip] (  0.00,  0.00) rectangle (151.77,144.54);
\definecolor{drawColor}{RGB}{255,255,255}
\definecolor{fillColor}{RGB}{255,255,255}

\path[draw=drawColor,line width= 0.6pt,line join=round,line cap=round,fill=fillColor] (  0.00,  0.00) rectangle (151.77,144.54);
\end{scope}
\begin{scope}
\path[clip] ( 28.02, 28.18) rectangle (146.27,125.77);
\definecolor{fillColor}{gray}{0.92}

\path[fill=fillColor] ( 28.02, 28.18) rectangle (146.27,125.77);
\definecolor{drawColor}{RGB}{255,255,255}

\path[draw=drawColor,line width= 0.3pt,line join=round] ( 28.02, 44.71) --
	(146.27, 44.71);

\path[draw=drawColor,line width= 0.3pt,line join=round] ( 28.02, 68.91) --
	(146.27, 68.91);

\path[draw=drawColor,line width= 0.3pt,line join=round] ( 28.02, 93.10) --
	(146.27, 93.10);

\path[draw=drawColor,line width= 0.3pt,line join=round] ( 28.02,117.30) --
	(146.27,117.30);

\path[draw=drawColor,line width= 0.3pt,line join=round] ( 30.57, 28.18) --
	( 30.57,125.77);

\path[draw=drawColor,line width= 0.3pt,line join=round] ( 87.15, 28.18) --
	( 87.15,125.77);

\path[draw=drawColor,line width= 0.3pt,line join=round] (143.72, 28.18) --
	(143.72,125.77);

\path[draw=drawColor,line width= 0.6pt,line join=round] ( 28.02, 32.62) --
	(146.27, 32.62);

\path[draw=drawColor,line width= 0.6pt,line join=round] ( 28.02, 56.81) --
	(146.27, 56.81);

\path[draw=drawColor,line width= 0.6pt,line join=round] ( 28.02, 81.01) --
	(146.27, 81.01);

\path[draw=drawColor,line width= 0.6pt,line join=round] ( 28.02,105.20) --
	(146.27,105.20);

\path[draw=drawColor,line width= 0.6pt,line join=round] ( 58.86, 28.18) --
	( 58.86,125.77);

\path[draw=drawColor,line width= 0.6pt,line join=round] (115.43, 28.18) --
	(115.43,125.77);
\definecolor{drawColor}{RGB}{0,0,0}
\definecolor{fillColor}{RGB}{72,118,255}

\path[draw=drawColor,line width= 0.6pt,line cap=rect,fill=fillColor] ( 78.66, 32.62) rectangle ( 84.32, 44.87);
\definecolor{fillColor}{RGB}{160,32,240}

\path[draw=drawColor,line width= 0.6pt,line cap=rect,fill=fillColor] ( 73.00, 32.62) rectangle ( 78.66, 55.92);
\definecolor{fillColor}{RGB}{255,0,255}

\path[draw=drawColor,line width= 0.6pt,line cap=rect,fill=fillColor] ( 67.34, 32.62) rectangle ( 73.00, 46.57);
\definecolor{fillColor}{RGB}{255,0,0}

\path[draw=drawColor,line width= 0.6pt,line cap=rect,fill=fillColor] ( 61.69, 32.62) rectangle ( 67.34, 58.75);
\definecolor{fillColor}{RGB}{255,165,0}

\path[draw=drawColor,line width= 0.6pt,line cap=rect,fill=fillColor] ( 56.03, 32.62) rectangle ( 61.69, 62.38);
\definecolor{fillColor}{RGB}{255,255,0}

\path[draw=drawColor,line width= 0.6pt,line cap=rect,fill=fillColor] ( 50.37, 32.62) rectangle ( 56.03, 65.68);
\definecolor{fillColor}{RGB}{173,255,47}

\path[draw=drawColor,line width= 0.6pt,line cap=rect,fill=fillColor] ( 44.71, 32.62) rectangle ( 50.37, 53.02);
\definecolor{fillColor}{RGB}{0,255,0}

\path[draw=drawColor,line width= 0.6pt,line cap=rect,fill=fillColor] ( 39.06, 32.62) rectangle ( 44.71, 59.47);
\definecolor{fillColor}{RGB}{0,0,255}

\path[draw=drawColor,line width= 0.6pt,line cap=rect,fill=fillColor] ( 33.40, 32.62) rectangle ( 39.06, 58.42);
\definecolor{fillColor}{RGB}{72,118,255}

\path[draw=drawColor,line width= 0.6pt,line cap=rect,fill=fillColor] (135.23, 32.62) rectangle (140.89, 68.10);
\definecolor{fillColor}{RGB}{160,32,240}

\path[draw=drawColor,line width= 0.6pt,line cap=rect,fill=fillColor] (129.58, 32.62) rectangle (135.23,121.33);
\definecolor{fillColor}{RGB}{255,0,255}

\path[draw=drawColor,line width= 0.6pt,line cap=rect,fill=fillColor] (123.92, 32.62) rectangle (129.58, 71.01);
\definecolor{fillColor}{RGB}{255,0,0}

\path[draw=drawColor,line width= 0.6pt,line cap=rect,fill=fillColor] (118.26, 32.62) rectangle (123.92, 80.52);
\definecolor{fillColor}{RGB}{255,165,0}

\path[draw=drawColor,line width= 0.6pt,line cap=rect,fill=fillColor] (112.60, 32.62) rectangle (118.26, 94.64);
\definecolor{fillColor}{RGB}{255,255,0}

\path[draw=drawColor,line width= 0.6pt,line cap=rect,fill=fillColor] (106.95, 32.62) rectangle (112.60,112.38);
\definecolor{fillColor}{RGB}{173,255,47}

\path[draw=drawColor,line width= 0.6pt,line cap=rect,fill=fillColor] (101.29, 32.62) rectangle (106.95, 82.30);
\definecolor{fillColor}{RGB}{0,255,0}

\path[draw=drawColor,line width= 0.6pt,line cap=rect,fill=fillColor] ( 95.63, 32.62) rectangle (101.29, 88.35);
\definecolor{fillColor}{RGB}{0,0,255}

\path[draw=drawColor,line width= 0.6pt,line cap=rect,fill=fillColor] ( 89.97, 32.62) rectangle ( 95.63,110.85);
\end{scope}
\begin{scope}
\path[clip] (  0.00,  0.00) rectangle (151.77,144.54);
\definecolor{drawColor}{RGB}{190,190,190}

\node[text=drawColor,rotate= 90.00,anchor=base,inner sep=0pt, outer sep=0pt, scale=  0.77] at ( 22.33, 32.62) {0.00};

\node[text=drawColor,rotate= 90.00,anchor=base,inner sep=0pt, outer sep=0pt, scale=  0.77] at ( 22.33, 56.81) {0.03};

\node[text=drawColor,rotate= 90.00,anchor=base,inner sep=0pt, outer sep=0pt, scale=  0.77] at ( 22.33, 81.01) {0.06};

\node[text=drawColor,rotate= 90.00,anchor=base,inner sep=0pt, outer sep=0pt, scale=  0.77] at ( 22.33,105.20) {0.09};
\end{scope}
\begin{scope}
\path[clip] (  0.00,  0.00) rectangle (151.77,144.54);
\definecolor{drawColor}{gray}{0.20}

\path[draw=drawColor,line width= 0.6pt,line join=round] ( 25.27, 32.62) --
	( 28.02, 32.62);

\path[draw=drawColor,line width= 0.6pt,line join=round] ( 25.27, 56.81) --
	( 28.02, 56.81);

\path[draw=drawColor,line width= 0.6pt,line join=round] ( 25.27, 81.01) --
	( 28.02, 81.01);

\path[draw=drawColor,line width= 0.6pt,line join=round] ( 25.27,105.20) --
	( 28.02,105.20);
\end{scope}
\begin{scope}
\path[clip] (  0.00,  0.00) rectangle (151.77,144.54);
\definecolor{drawColor}{gray}{0.20}

\path[draw=drawColor,line width= 0.6pt,line join=round] ( 58.86, 25.43) --
	( 58.86, 28.18);

\path[draw=drawColor,line width= 0.6pt,line join=round] (115.43, 25.43) --
	(115.43, 28.18);
\end{scope}
\begin{scope}
\path[clip] (  0.00,  0.00) rectangle (151.77,144.54);
\definecolor{drawColor}{RGB}{190,190,190}

\node[text=drawColor,anchor=base,inner sep=0pt, outer sep=0pt, scale=  0.78] at ( 58.86, 17.07) {2};

\node[text=drawColor,anchor=base,inner sep=0pt, outer sep=0pt, scale=  0.78] at (115.43, 17.07) {3};
\end{scope}
\begin{scope}
\path[clip] (  0.00,  0.00) rectangle (151.77,144.54);
\definecolor{drawColor}{RGB}{0,0,0}

\node[text=drawColor,anchor=base,inner sep=0pt, outer sep=0pt, scale=  0.91] at ( 87.15,  7.27) {Number of Objects};
\end{scope}
\begin{scope}
\path[clip] (  0.00,  0.00) rectangle (151.77,144.54);
\definecolor{drawColor}{RGB}{0,0,0}

\node[text=drawColor,anchor=base,inner sep=0pt, outer sep=0pt, scale=  0.88] at ( 87.15,132.98) {Hopkins};
\end{scope}
\end{tikzpicture}} & \hspace{-.5cm}
\Scale[.55]{% Created by tikzDevice version 0.12.3 on 2020-05-26 15:07:33
% !TEX encoding = UTF-8 Unicode
\begin{tikzpicture}[x=1pt,y=1pt]
\definecolor{fillColor}{RGB}{255,255,255}
\path[use as bounding box,fill=fillColor,fill opacity=0.00] (0,0) rectangle (151.77,144.54);
\begin{scope}
\path[clip] (  0.00,  0.00) rectangle (151.77,144.54);
\definecolor{drawColor}{RGB}{255,255,255}
\definecolor{fillColor}{RGB}{255,255,255}

\path[draw=drawColor,line width= 0.6pt,line join=round,line cap=round,fill=fillColor] (  0.00, -0.00) rectangle (151.77,144.54);
\end{scope}
\begin{scope}
\path[clip] ( 28.02, 30.80) rectangle (146.27,125.77);
\definecolor{fillColor}{gray}{0.92}

\path[fill=fillColor] ( 28.02, 30.80) rectangle (146.27,125.77);
\definecolor{drawColor}{RGB}{255,255,255}

\path[draw=drawColor,line width= 0.3pt,line join=round] ( 28.02,124.65) --
	(146.27,124.65);

\path[draw=drawColor,line width= 0.3pt,line join=round] ( 28.02, 99.07) --
	(146.27, 99.07);

\path[draw=drawColor,line width= 0.3pt,line join=round] ( 28.02, 73.49) --
	(146.27, 73.49);

\path[draw=drawColor,line width= 0.3pt,line join=round] ( 28.02, 51.11) --
	(146.27, 51.11);

\path[draw=drawColor,line width= 0.3pt,line join=round] ( 45.34, 30.80) --
	( 45.34,125.77);

\path[draw=drawColor,line width= 0.3pt,line join=round] ( 69.23, 30.80) --
	( 69.23,125.77);

\path[draw=drawColor,line width= 0.3pt,line join=round] ( 93.12, 30.80) --
	( 93.12,125.77);

\path[draw=drawColor,line width= 0.3pt,line join=round] (117.00, 30.80) --
	(117.00,125.77);

\path[draw=drawColor,line width= 0.3pt,line join=round] (134.92, 30.80) --
	(134.92,125.77);

\path[draw=drawColor,line width= 0.6pt,line join=round] ( 28.02,111.86) --
	(146.27,111.86);

\path[draw=drawColor,line width= 0.6pt,line join=round] ( 28.02, 86.28) --
	(146.27, 86.28);

\path[draw=drawColor,line width= 0.6pt,line join=round] ( 28.02, 60.70) --
	(146.27, 60.70);

\path[draw=drawColor,line width= 0.6pt,line join=round] ( 28.02, 41.51) --
	(146.27, 41.51);

\path[draw=drawColor,line width= 0.6pt,line join=round] ( 33.40, 30.80) --
	( 33.40,125.77);

\path[draw=drawColor,line width= 0.6pt,line join=round] ( 57.29, 30.80) --
	( 57.29,125.77);

\path[draw=drawColor,line width= 0.6pt,line join=round] ( 81.17, 30.80) --
	( 81.17,125.77);

\path[draw=drawColor,line width= 0.6pt,line join=round] (105.06, 30.80) --
	(105.06,125.77);

\path[draw=drawColor,line width= 0.6pt,line join=round] (128.95, 30.80) --
	(128.95,125.77);

\path[draw=drawColor,line width= 0.6pt,line join=round] (140.89, 30.80) --
	(140.89,125.77);
\definecolor{drawColor}{RGB}{0,0,255}

\path[draw=drawColor,line width= 2.3pt,line join=round] ( 33.40, 48.86) --
	( 45.34, 49.06) --
	( 57.29, 48.80) --
	( 69.23, 51.36) --
	( 81.17, 54.06) --
	( 93.12, 53.51) --
	(105.06, 55.19) --
	(117.00, 54.30) --
	(128.95, 54.34) --
	(140.89,115.06);

\path[] ( 33.40, 35.12) --
	( 45.34, 35.12) --
	( 57.29, 35.12) --
	( 69.23, 35.12) --
	( 81.17, 35.12) --
	( 93.12, 35.12) --
	(105.06, 35.12) --
	(117.00, 35.12) --
	(128.95, 35.12) --
	(140.89, 35.12);

\path[] ( 33.40, 35.12) --
	( 45.34, 35.12) --
	( 57.29, 35.12) --
	( 69.23, 35.12) --
	( 81.17, 35.12) --
	( 93.12, 35.12) --
	(105.06, 35.12) --
	(117.00, 35.12) --
	(128.95, 35.12) --
	(140.89, 35.12);

\path[] ( 33.40, 35.12) --
	( 45.34, 35.12) --
	( 57.29, 35.12) --
	( 69.23, 35.12) --
	( 81.17, 35.12) --
	( 93.12, 35.12) --
	(105.06, 35.12) --
	(117.00, 35.12) --
	(128.95, 35.12) --
	(140.89, 35.12);

\path[] ( 33.40, 35.12) --
	( 45.34, 35.12) --
	( 57.29, 35.12) --
	( 69.23, 35.12) --
	( 81.17, 35.12) --
	( 93.12, 35.12) --
	(105.06, 35.12) --
	(117.00, 35.12) --
	(128.95, 35.12) --
	(140.89, 35.12);

\path[] ( 33.40, 35.12) --
	( 45.34, 35.12) --
	( 57.29, 35.12) --
	( 69.23, 35.12) --
	( 81.17, 35.12) --
	( 93.12, 35.12) --
	(105.06, 35.12) --
	(117.00, 35.12) --
	(128.95, 35.12) --
	(140.89, 35.12);

\path[] ( 33.40, 35.12) --
	( 45.34, 35.12) --
	( 57.29, 35.12) --
	( 69.23, 35.12) --
	( 81.17, 35.12) --
	( 93.12, 35.12) --
	(105.06, 35.12) --
	(117.00, 35.12) --
	(128.95, 35.12) --
	(140.89, 35.12);

\path[] ( 33.40, 35.12) --
	( 45.34, 35.12) --
	( 57.29, 35.12) --
	( 69.23, 35.12) --
	( 81.17, 35.12) --
	( 93.12, 35.12) --
	(105.06, 35.12) --
	(117.00, 35.12) --
	(128.95, 35.12) --
	(140.89, 35.12);
\definecolor{drawColor}{RGB}{255,105,180}

\path[draw=drawColor,line width= 1.1pt,line join=round] ( 33.40, 54.30) --
	( 45.34, 59.10) --
	( 57.29, 67.09) --
	( 69.23, 70.93) --
	( 81.17, 79.88) --
	( 93.12, 96.51) --
	(105.06, 95.03) --
	(117.00, 98.65) --
	(128.95,102.83) --
	(140.89,121.45);
\definecolor{drawColor}{RGB}{218,165,32}

\path[draw=drawColor,line width= 1.1pt,dash pattern=on 4pt off 4pt ,line join=round] ( 33.40, 47.91) --
	( 45.34, 56.86) --
	( 57.29, 60.70) --
	( 69.23, 70.29) --
	( 81.17, 76.69) --
	( 93.12, 93.95) --
	(105.06, 93.95) --
	(117.00, 98.65) --
	(128.95,102.52) --
	(140.89,118.25);
\definecolor{drawColor}{RGB}{0,255,255}

\path[draw=drawColor,line width= 1.1pt,dash pattern=on 2pt off 2pt on 6pt off 2pt ,line join=round] ( 33.40, 47.91) --
	( 45.34, 57.50) --
	( 57.29, 60.70) --
	( 69.23, 69.01) --
	( 81.17, 73.49) --
	( 93.12, 92.67) --
	(105.06, 94.98) --
	(117.00, 98.56) --
	(128.95,102.27) --
	(140.89,115.06);
\definecolor{drawColor}{RGB}{0,0,0}

\path[draw=drawColor,line width= 1.1pt,dash pattern=on 4pt off 4pt ,line join=round] ( 33.40, 51.11) --
	( 45.34, 56.86) --
	( 57.29, 62.30) --
	( 69.23, 69.65) --
	( 81.17, 73.49) --
	( 93.12, 94.27) --
	(105.06, 94.11) --
	(117.00, 98.43) --
	(128.95,102.65) --
	(140.89,115.06);
\definecolor{drawColor}{RGB}{102,205,170}

\path[draw=drawColor,line width= 1.1pt,dash pattern=on 2pt off 2pt on 6pt off 2pt ,line join=round] ( 33.40, 41.51) --
	( 45.34, 41.51) --
	( 57.29, 41.51) --
	( 69.23, 54.30) --
	( 81.17, 62.62) --
	( 93.12, 97.47) --
	(105.06,107.06) --
	(117.00,105.46) --
	(128.95,113.46) --
	(140.89,121.45);
\definecolor{drawColor}{RGB}{0,100,0}

\path[draw=drawColor,line width= 1.1pt,dash pattern=on 2pt off 2pt on 6pt off 2pt ,line join=round] ( 33.40, 48.86) --
	( 45.34, 49.06) --
	( 57.29, 48.80) --
	( 69.23, 51.36) --
	( 81.17, 56.72) --
	( 93.12, 57.04) --
	(105.06, 56.51) --
	(117.00, 57.81) --
	(128.95,102.83) --
	(140.89,113.46);

\path[] ( 33.40, 35.12) --
	( 45.34, 35.12) --
	( 57.29, 35.12) --
	( 69.23, 35.12) --
	( 81.17, 35.12) --
	( 93.12, 35.12) --
	(105.06, 35.12) --
	(117.00, 35.12) --
	(128.95, 35.12) --
	(140.89, 35.12);

\path[] ( 33.40, 35.12) --
	( 45.34, 35.12) --
	( 57.29, 35.12) --
	( 69.23, 35.12) --
	( 81.17, 35.12) --
	( 93.12, 35.12) --
	(105.06, 35.12) --
	(117.00, 35.12) --
	(128.95, 35.12) --
	(140.89, 35.12);
\end{scope}
\begin{scope}
\path[clip] (  0.00,  0.00) rectangle (151.77,144.54);
\definecolor{drawColor}{RGB}{190,190,190}

\node[text=drawColor,rotate= 90.00,anchor=base,inner sep=0pt, outer sep=0pt, scale=  0.77] at ( 22.33,111.86) {0.6};

\node[text=drawColor,rotate= 90.00,anchor=base,inner sep=0pt, outer sep=0pt, scale=  0.77] at ( 22.33, 86.28) {0.4};

\node[text=drawColor,rotate= 90.00,anchor=base,inner sep=0pt, outer sep=0pt, scale=  0.77] at ( 22.33, 60.70) {0.2};

\node[text=drawColor,rotate= 90.00,anchor=base,inner sep=0pt, outer sep=0pt, scale=  0.77] at ( 22.33, 41.51) {0.05};
\end{scope}
\begin{scope}
\path[clip] (  0.00,  0.00) rectangle (151.77,144.54);
\definecolor{drawColor}{gray}{0.20}

\path[draw=drawColor,line width= 0.6pt,line join=round] ( 25.27,111.86) --
	( 28.02,111.86);

\path[draw=drawColor,line width= 0.6pt,line join=round] ( 25.27, 86.28) --
	( 28.02, 86.28);

\path[draw=drawColor,line width= 0.6pt,line join=round] ( 25.27, 60.70) --
	( 28.02, 60.70);

\path[draw=drawColor,line width= 0.6pt,line join=round] ( 25.27, 41.51) --
	( 28.02, 41.51);
\end{scope}
\begin{scope}
\path[clip] (  0.00,  0.00) rectangle (151.77,144.54);
\definecolor{drawColor}{gray}{0.20}

\path[draw=drawColor,line width= 0.6pt,line join=round] ( 33.40, 28.05) --
	( 33.40, 30.80);

\path[draw=drawColor,line width= 0.6pt,line join=round] ( 57.29, 28.05) --
	( 57.29, 30.80);

\path[draw=drawColor,line width= 0.6pt,line join=round] ( 81.17, 28.05) --
	( 81.17, 30.80);

\path[draw=drawColor,line width= 0.6pt,line join=round] (105.06, 28.05) --
	(105.06, 30.80);

\path[draw=drawColor,line width= 0.6pt,line join=round] (128.95, 28.05) --
	(128.95, 30.80);

\path[draw=drawColor,line width= 0.6pt,line join=round] (140.89, 28.05) --
	(140.89, 30.80);
\end{scope}
\begin{scope}
\path[clip] (  0.00,  0.00) rectangle (151.77,144.54);
\definecolor{drawColor}{RGB}{190,190,190}

\node[text=drawColor,anchor=base,inner sep=0pt, outer sep=0pt, scale=  0.78] at ( 33.40, 19.70) {.1};

\node[text=drawColor,anchor=base,inner sep=0pt, outer sep=0pt, scale=  0.78] at ( 57.29, 19.70) {.3};

\node[text=drawColor,anchor=base,inner sep=0pt, outer sep=0pt, scale=  0.78] at ( 81.17, 19.70) {.5};

\node[text=drawColor,anchor=base,inner sep=0pt, outer sep=0pt, scale=  0.78] at (105.06, 19.70) {.7};

\node[text=drawColor,anchor=base,inner sep=0pt, outer sep=0pt, scale=  0.78] at (128.95, 19.70) {.9};

\node[text=drawColor,anchor=base,inner sep=0pt, outer sep=0pt, scale=  0.78] at (140.89, 19.70) {.99};
\end{scope}
\begin{scope}
\path[clip] (  0.00,  0.00) rectangle (151.77,144.54);
\definecolor{drawColor}{RGB}{0,0,0}

\node[text=drawColor,anchor=base,inner sep=0pt, outer sep=0pt, scale=  1.21] at ( 87.15,  7.85) {\underline{\textbf{MISSING DATA}}};
\end{scope}
\begin{scope}
\path[clip] (  0.00,  0.00) rectangle (151.77,144.54);
\definecolor{drawColor}{RGB}{0,0,0}

\node[text=drawColor,anchor=base,inner sep=0pt, outer sep=0pt, scale=  0.88] at ( 87.15,132.98) {Hopkins};
\end{scope}
\end{tikzpicture}}
\\
\Scale[.55]{\input{Figures/FiguresSims/NoiseError}} & \hspace{-.5cm}
\Scale[.55]{% Created by tikzDevice version 0.12.3 on 2020-05-26 00:24:07
% !TEX encoding = UTF-8 Unicode
\begin{tikzpicture}[x=1pt,y=1pt]
\definecolor{fillColor}{RGB}{255,255,255}
\path[use as bounding box,fill=fillColor,fill opacity=0.00] (0,0) rectangle (151.77,144.54);
\begin{scope}
\path[clip] (  0.00,  0.00) rectangle (151.77,144.54);
\definecolor{drawColor}{RGB}{255,255,255}
\definecolor{fillColor}{RGB}{255,255,255}

\path[draw=drawColor,line width= 0.6pt,line join=round,line cap=round,fill=fillColor] (  0.00, -0.00) rectangle (151.77,144.54);
\end{scope}
\begin{scope}
\path[clip] ( 28.02, 30.80) rectangle (146.27,125.77);
\definecolor{fillColor}{gray}{0.92}

\path[fill=fillColor] ( 28.02, 30.80) rectangle (146.27,125.77);
\definecolor{drawColor}{RGB}{255,255,255}

\path[draw=drawColor,line width= 0.3pt,line join=round] ( 28.02,109.27) --
	(146.27,109.27);

\path[draw=drawColor,line width= 0.3pt,line join=round] ( 28.02, 88.08) --
	(146.27, 88.08);

\path[draw=drawColor,line width= 0.3pt,line join=round] ( 28.02, 66.90) --
	(146.27, 66.90);

\path[draw=drawColor,line width= 0.3pt,line join=round] ( 28.02, 45.71) --
	(146.27, 45.71);

\path[draw=drawColor,line width= 0.3pt,line join=round] ( 41.08, 30.80) --
	( 41.08,125.77);

\path[draw=drawColor,line width= 0.3pt,line join=round] ( 56.43, 30.80) --
	( 56.43,125.77);

\path[draw=drawColor,line width= 0.3pt,line join=round] ( 71.79, 30.80) --
	( 71.79,125.77);

\path[draw=drawColor,line width= 0.3pt,line join=round] ( 87.15, 30.80) --
	( 87.15,125.77);

\path[draw=drawColor,line width= 0.3pt,line join=round] (102.50, 30.80) --
	(102.50,125.77);

\path[draw=drawColor,line width= 0.3pt,line join=round] (117.86, 30.80) --
	(117.86,125.77);

\path[draw=drawColor,line width= 0.3pt,line join=round] (133.21, 30.80) --
	(133.21,125.77);

\path[draw=drawColor,line width= 0.6pt,line join=round] ( 28.02,119.86) --
	(146.27,119.86);

\path[draw=drawColor,line width= 0.6pt,line join=round] ( 28.02, 98.68) --
	(146.27, 98.68);

\path[draw=drawColor,line width= 0.6pt,line join=round] ( 28.02, 77.49) --
	(146.27, 77.49);

\path[draw=drawColor,line width= 0.6pt,line join=round] ( 28.02, 56.30) --
	(146.27, 56.30);

\path[draw=drawColor,line width= 0.6pt,line join=round] ( 28.02, 35.12) --
	(146.27, 35.12);

\path[draw=drawColor,line width= 0.6pt,line join=round] ( 33.40, 30.80) --
	( 33.40,125.77);

\path[draw=drawColor,line width= 0.6pt,line join=round] ( 48.75, 30.80) --
	( 48.75,125.77);

\path[draw=drawColor,line width= 0.6pt,line join=round] ( 64.11, 30.80) --
	( 64.11,125.77);

\path[draw=drawColor,line width= 0.6pt,line join=round] ( 79.47, 30.80) --
	( 79.47,125.77);

\path[draw=drawColor,line width= 0.6pt,line join=round] ( 94.82, 30.80) --
	( 94.82,125.77);

\path[draw=drawColor,line width= 0.6pt,line join=round] (110.18, 30.80) --
	(110.18,125.77);

\path[draw=drawColor,line width= 0.6pt,line join=round] (125.54, 30.80) --
	(125.54,125.77);

\path[draw=drawColor,line width= 0.6pt,line join=round] (140.89, 30.80) --
	(140.89,125.77);
\definecolor{drawColor}{RGB}{0,0,255}

\path[draw=drawColor,line width= 2.3pt,line join=round] ( 33.40, 36.85) --
	( 48.75, 36.44) --
	( 64.11, 39.36) --
	( 79.47, 44.65) --
	( 94.82, 48.50) --
	(110.18, 43.06) --
	(125.54, 43.06) --
	(140.89,109.27);

\path[] ( 33.40, 35.12) --
	( 48.75, 35.12) --
	( 64.11, 35.12) --
	( 79.47, 35.12) --
	( 94.82, 35.12) --
	(110.18, 35.12) --
	(125.54, 35.12) --
	(140.89, 35.12);

\path[] ( 33.40, 35.12) --
	( 48.75, 35.12) --
	( 64.11, 35.12) --
	( 79.47, 35.12) --
	( 94.82, 35.12) --
	(110.18, 35.12) --
	(125.54, 35.12) --
	(140.89, 35.12);

\path[] ( 33.40, 35.12) --
	( 48.75, 35.12) --
	( 64.11, 35.12) --
	( 79.47, 35.12) --
	( 94.82, 35.12) --
	(110.18, 35.12) --
	(125.54, 35.12) --
	(140.89, 35.12);

\path[] ( 33.40, 35.12) --
	( 48.75, 35.12) --
	( 64.11, 35.12) --
	( 79.47, 35.12) --
	( 94.82, 35.12) --
	(110.18, 35.12) --
	(125.54, 35.12) --
	(140.89, 35.12);

\path[] ( 33.40, 35.12) --
	( 48.75, 35.12) --
	( 64.11, 35.12) --
	( 79.47, 35.12) --
	( 94.82, 35.12) --
	(110.18, 35.12) --
	(125.54, 35.12) --
	(140.89, 35.12);

\path[] ( 33.40, 35.12) --
	( 48.75, 35.12) --
	( 64.11, 35.12) --
	( 79.47, 35.12) --
	( 94.82, 35.12) --
	(110.18, 35.12) --
	(125.54, 35.12) --
	(140.89, 35.12);

\path[] ( 33.40, 35.12) --
	( 48.75, 35.12) --
	( 64.11, 35.12) --
	( 79.47, 35.12) --
	( 94.82, 35.12) --
	(110.18, 35.12) --
	(125.54, 35.12) --
	(140.89, 35.12);
\definecolor{drawColor}{RGB}{255,105,180}

\path[draw=drawColor,line width= 1.1pt,line join=round] ( 33.40, 39.09) --
	( 48.75, 44.39) --
	( 64.11, 44.39) --
	( 79.47, 74.84) --
	( 94.82,100.00) --
	(110.18,106.62) --
	(125.54,111.92) --
	(140.89,117.74);
\definecolor{drawColor}{RGB}{218,165,32}

\path[draw=drawColor,line width= 1.1pt,dash pattern=on 4pt off 4pt ,line join=round] ( 33.40, 39.09) --
	( 48.75, 39.09) --
	( 64.11, 39.09) --
	( 79.47, 57.63) --
	( 94.82,107.95) --
	(110.18,106.62) --
	(125.54,111.92) --
	(140.89,121.45);
\definecolor{drawColor}{RGB}{0,255,255}

\path[draw=drawColor,line width= 1.1pt,dash pattern=on 2pt off 2pt on 6pt off 2pt ,line join=round] ( 33.40, 36.44) --
	( 48.75, 36.44) --
	( 64.11, 36.44) --
	( 79.47, 52.33) --
	( 94.82,105.30) --
	(110.18,105.30) --
	(125.54,106.62) --
	(140.89,118.27);
\definecolor{drawColor}{RGB}{0,0,0}

\path[draw=drawColor,line width= 1.1pt,dash pattern=on 4pt off 4pt ,line join=round] ( 33.40, 39.09) --
	( 48.75, 39.09) --
	( 64.11, 39.09) --
	( 79.47, 57.63) --
	( 94.82,118.54) --
	(110.18,113.24) --
	(125.54,113.24) --
	(140.89,113.51);
\definecolor{drawColor}{RGB}{102,205,170}

\path[draw=drawColor,line width= 1.1pt,dash pattern=on 2pt off 2pt on 6pt off 2pt ,line join=round] ( 33.40, 35.12) --
	( 48.75, 35.12) --
	( 64.11, 35.12) --
	( 79.47, 69.55) --
	( 94.82,109.27) --
	(110.18,100.00) --
	(125.54,110.59) --
	(140.89,117.74);

\path[] ( 33.40, 35.12) --
	( 48.75, 35.12) --
	( 64.11, 35.12) --
	( 79.47, 35.12) --
	( 94.82, 35.12) --
	(110.18, 35.12) --
	(125.54, 35.12) --
	(140.89, 35.12);

\path[] ( 33.40, 35.12) --
	( 48.75, 35.12) --
	( 64.11, 35.12) --
	( 79.47, 35.12) --
	( 94.82, 35.12) --
	(110.18, 35.12) --
	(125.54, 35.12) --
	(140.89, 35.12);

\path[] ( 33.40, 35.12) --
	( 48.75, 35.12) --
	( 64.11, 35.12) --
	( 79.47, 35.12) --
	( 94.82, 35.12) --
	(110.18, 35.12) --
	(125.54, 35.12) --
	(140.89, 35.12);
\end{scope}
\begin{scope}
\path[clip] (  0.00,  0.00) rectangle (151.77,144.54);
\definecolor{drawColor}{RGB}{190,190,190}

\node[text=drawColor,rotate= 90.00,anchor=base,inner sep=0pt, outer sep=0pt, scale=  0.77] at ( 22.33,119.86) {0.8};

\node[text=drawColor,rotate= 90.00,anchor=base,inner sep=0pt, outer sep=0pt, scale=  0.77] at ( 22.33, 98.68) {0.6};

\node[text=drawColor,rotate= 90.00,anchor=base,inner sep=0pt, outer sep=0pt, scale=  0.77] at ( 22.33, 77.49) {0.4};

\node[text=drawColor,rotate= 90.00,anchor=base,inner sep=0pt, outer sep=0pt, scale=  0.77] at ( 22.33, 56.30) {0.2};

\node[text=drawColor,rotate= 90.00,anchor=base,inner sep=0pt, outer sep=0pt, scale=  0.77] at ( 22.33, 35.12) {0.0};
\end{scope}
\begin{scope}
\path[clip] (  0.00,  0.00) rectangle (151.77,144.54);
\definecolor{drawColor}{gray}{0.20}

\path[draw=drawColor,line width= 0.6pt,line join=round] ( 25.27,119.86) --
	( 28.02,119.86);

\path[draw=drawColor,line width= 0.6pt,line join=round] ( 25.27, 98.68) --
	( 28.02, 98.68);

\path[draw=drawColor,line width= 0.6pt,line join=round] ( 25.27, 77.49) --
	( 28.02, 77.49);

\path[draw=drawColor,line width= 0.6pt,line join=round] ( 25.27, 56.30) --
	( 28.02, 56.30);

\path[draw=drawColor,line width= 0.6pt,line join=round] ( 25.27, 35.12) --
	( 28.02, 35.12);
\end{scope}
\begin{scope}
\path[clip] (  0.00,  0.00) rectangle (151.77,144.54);
\definecolor{drawColor}{gray}{0.20}

\path[draw=drawColor,line width= 0.6pt,line join=round] ( 33.40, 28.05) --
	( 33.40, 30.80);

\path[draw=drawColor,line width= 0.6pt,line join=round] ( 48.75, 28.05) --
	( 48.75, 30.80);

\path[draw=drawColor,line width= 0.6pt,line join=round] ( 64.11, 28.05) --
	( 64.11, 30.80);

\path[draw=drawColor,line width= 0.6pt,line join=round] ( 79.47, 28.05) --
	( 79.47, 30.80);

\path[draw=drawColor,line width= 0.6pt,line join=round] ( 94.82, 28.05) --
	( 94.82, 30.80);

\path[draw=drawColor,line width= 0.6pt,line join=round] (110.18, 28.05) --
	(110.18, 30.80);

\path[draw=drawColor,line width= 0.6pt,line join=round] (125.54, 28.05) --
	(125.54, 30.80);

\path[draw=drawColor,line width= 0.6pt,line join=round] (140.89, 28.05) --
	(140.89, 30.80);
\end{scope}
\begin{scope}
\path[clip] (  0.00,  0.00) rectangle (151.77,144.54);
\definecolor{drawColor}{RGB}{190,190,190}

\node[text=drawColor,anchor=base,inner sep=0pt, outer sep=0pt, scale=  0.78] at ( 33.40, 19.70) {.3};

\node[text=drawColor,anchor=base,inner sep=0pt, outer sep=0pt, scale=  0.78] at ( 48.75, 19.70) {.4};

\node[text=drawColor,anchor=base,inner sep=0pt, outer sep=0pt, scale=  0.78] at ( 64.11, 19.70) {.5};

\node[text=drawColor,anchor=base,inner sep=0pt, outer sep=0pt, scale=  0.78] at ( 79.47, 19.70) {.6};

\node[text=drawColor,anchor=base,inner sep=0pt, outer sep=0pt, scale=  0.78] at ( 94.82, 19.70) {.7};

\node[text=drawColor,anchor=base,inner sep=0pt, outer sep=0pt, scale=  0.78] at (110.18, 19.70) {.8};

\node[text=drawColor,anchor=base,inner sep=0pt, outer sep=0pt, scale=  0.78] at (125.54, 19.70) {.9};

\node[text=drawColor,anchor=base,inner sep=0pt, outer sep=0pt, scale=  0.78] at (140.89, 19.70) {.94};
\end{scope}
\begin{scope}
\path[clip] (  0.00,  0.00) rectangle (151.77,144.54);
\definecolor{drawColor}{RGB}{0,0,0}

\node[text=drawColor,anchor=base,inner sep=0pt, outer sep=0pt, scale=  1.21] at ( 87.15,  7.85) {\underline{\textbf{MISSING DATA}}};
\end{scope}
\begin{scope}
\path[clip] (  0.00,  0.00) rectangle (151.77,144.54);
\definecolor{drawColor}{RGB}{0,0,0}

\node[text=drawColor,anchor=base,inner sep=0pt, outer sep=0pt, scale=  0.88] at ( 87.15,132.98) {$(n_k=20)$};
\end{scope}
\end{tikzpicture}} & \hspace{-.5cm}
\Scale[.55]{\input{Figures/FiguresSims/SamplingError50.tex}} & \hspace{-.25cm}
\Scale[.55]{\input{Figures/FiguresData/MNISTFull.tex}} & \hspace{-.5cm}
\Scale[.55]{% Created by tikzDevice version 0.12.3 on 2020-05-26 15:12:17
% !TEX encoding = UTF-8 Unicode
\begin{tikzpicture}[x=1pt,y=1pt]
\definecolor{fillColor}{RGB}{255,255,255}
\path[use as bounding box,fill=fillColor,fill opacity=0.00] (0,0) rectangle (151.77,144.54);
\begin{scope}
\path[clip] (  0.00,  0.00) rectangle (151.77,144.54);
\definecolor{drawColor}{RGB}{255,255,255}
\definecolor{fillColor}{RGB}{255,255,255}

\path[draw=drawColor,line width= 0.6pt,line join=round,line cap=round,fill=fillColor] (  0.00, -0.00) rectangle (151.77,144.54);
\end{scope}
\begin{scope}
\path[clip] ( 28.02, 30.80) rectangle (146.27,125.77);
\definecolor{fillColor}{gray}{0.92}

\path[fill=fillColor] ( 28.02, 30.80) rectangle (146.27,125.77);
\definecolor{drawColor}{RGB}{255,255,255}

\path[draw=drawColor,line width= 0.3pt,line join=round] ( 28.02, 47.11) --
	(146.27, 47.11);

\path[draw=drawColor,line width= 0.3pt,line join=round] ( 28.02, 71.09) --
	(146.27, 71.09);

\path[draw=drawColor,line width= 0.3pt,line join=round] ( 28.02, 95.07) --
	(146.27, 95.07);

\path[draw=drawColor,line width= 0.3pt,line join=round] ( 28.02,119.05) --
	(146.27,119.05);

\path[draw=drawColor,line width= 0.3pt,line join=round] ( 45.34, 30.80) --
	( 45.34,125.77);

\path[draw=drawColor,line width= 0.3pt,line join=round] ( 69.23, 30.80) --
	( 69.23,125.77);

\path[draw=drawColor,line width= 0.3pt,line join=round] ( 93.12, 30.80) --
	( 93.12,125.77);

\path[draw=drawColor,line width= 0.3pt,line join=round] (117.00, 30.80) --
	(117.00,125.77);

\path[draw=drawColor,line width= 0.3pt,line join=round] (134.92, 30.80) --
	(134.92,125.77);

\path[draw=drawColor,line width= 0.6pt,line join=round] ( 28.02, 35.12) --
	(146.27, 35.12);

\path[draw=drawColor,line width= 0.6pt,line join=round] ( 28.02, 59.10) --
	(146.27, 59.10);

\path[draw=drawColor,line width= 0.6pt,line join=round] ( 28.02, 83.08) --
	(146.27, 83.08);

\path[draw=drawColor,line width= 0.6pt,line join=round] ( 28.02,107.06) --
	(146.27,107.06);

\path[draw=drawColor,line width= 0.6pt,line join=round] ( 33.40, 30.80) --
	( 33.40,125.77);

\path[draw=drawColor,line width= 0.6pt,line join=round] ( 57.29, 30.80) --
	( 57.29,125.77);

\path[draw=drawColor,line width= 0.6pt,line join=round] ( 81.17, 30.80) --
	( 81.17,125.77);

\path[draw=drawColor,line width= 0.6pt,line join=round] (105.06, 30.80) --
	(105.06,125.77);

\path[draw=drawColor,line width= 0.6pt,line join=round] (128.95, 30.80) --
	(128.95,125.77);

\path[draw=drawColor,line width= 0.6pt,line join=round] (140.89, 30.80) --
	(140.89,125.77);
\definecolor{drawColor}{RGB}{0,0,255}

\path[draw=drawColor,line width= 2.3pt,line join=round] ( 33.40, 40.04) --
	( 45.34, 43.99) --
	( 57.29, 42.43) --
	( 69.23, 45.49) --
	( 81.17, 47.89) --
	( 93.12, 44.47) --
	(105.06, 50.83) --
	(117.00, 44.17) --
	(128.95, 44.71) --
	(140.89,119.05);

\path[] ( 33.40, 11.14) --
	( 45.34, 11.14) --
	( 57.29, 11.14) --
	( 69.23, 11.14) --
	( 81.17, 11.14) --
	( 93.12, 11.14) --
	(105.06, 11.14) --
	(117.00, 11.14) --
	(128.95, 11.14) --
	(140.89, 11.14);

\path[] ( 33.40, 11.14) --
	( 45.34, 11.14) --
	( 57.29, 11.14) --
	( 69.23, 11.14) --
	( 81.17, 11.14) --
	( 93.12, 11.14) --
	(105.06, 11.14) --
	(117.00, 11.14) --
	(128.95, 11.14) --
	(140.89, 11.14);

\path[] ( 33.40, 11.14) --
	( 45.34, 11.14) --
	( 57.29, 11.14) --
	( 69.23, 11.14) --
	( 81.17, 11.14) --
	( 93.12, 11.14) --
	(105.06, 11.14) --
	(117.00, 11.14) --
	(128.95, 11.14) --
	(140.89, 11.14);

\path[] ( 33.40, 11.14) --
	( 45.34, 11.14) --
	( 57.29, 11.14) --
	( 69.23, 11.14) --
	( 81.17, 11.14) --
	( 93.12, 11.14) --
	(105.06, 11.14) --
	(117.00, 11.14) --
	(128.95, 11.14) --
	(140.89, 11.14);

\path[] ( 33.40, 11.14) --
	( 45.34, 11.14) --
	( 57.29, 11.14) --
	( 69.23, 11.14) --
	( 81.17, 11.14) --
	( 93.12, 11.14) --
	(105.06, 11.14) --
	(117.00, 11.14) --
	(128.95, 11.14) --
	(140.89, 11.14);

\path[] ( 33.40, 11.14) --
	( 45.34, 11.14) --
	( 57.29, 11.14) --
	( 69.23, 11.14) --
	( 81.17, 11.14) --
	( 93.12, 11.14) --
	(105.06, 11.14) --
	(117.00, 11.14) --
	(128.95, 11.14) --
	(140.89, 11.14);

\path[] ( 33.40, 11.14) --
	( 45.34, 11.14) --
	( 57.29, 11.14) --
	( 69.23, 11.14) --
	( 81.17, 11.14) --
	( 93.12, 11.14) --
	(105.06, 11.14) --
	(117.00, 11.14) --
	(128.95, 11.14) --
	(140.89, 11.14);
\definecolor{drawColor}{RGB}{255,105,180}

\path[draw=drawColor,line width= 1.1pt,line join=round] ( 33.40, 90.20) --
	( 45.34, 89.74) --
	( 57.29, 90.22) --
	( 69.23, 82.42) --
	( 81.17, 78.29) --
	( 93.12, 76.49) --
	(105.06, 70.91) --
	(117.00, 73.49) --
	(128.95, 80.62) --
	(140.89,119.05);
\definecolor{drawColor}{RGB}{218,165,32}

\path[draw=drawColor,line width= 1.1pt,dash pattern=on 4pt off 4pt ,line join=round] ( 33.40, 89.68) --
	( 45.34, 91.48) --
	( 57.29, 90.28) --
	( 69.23, 80.08) --
	( 81.17, 76.19) --
	( 93.12, 80.38) --
	(105.06, 77.39) --
	(117.00, 71.39) --
	(128.95, 89.08) --
	(140.89,119.05);
\definecolor{drawColor}{RGB}{0,255,255}

\path[draw=drawColor,line width= 1.1pt,dash pattern=on 2pt off 2pt on 6pt off 2pt ,line join=round] ( 33.40, 78.88) --
	( 45.34, 88.68) --
	( 57.29, 88.03) --
	( 69.23, 77.84) --
	( 81.17, 75.89) --
	( 93.12, 78.74) --
	(105.06, 73.49) --
	(117.00, 72.29) --
	(128.95, 79.33) --
	(140.89,119.05);
\definecolor{drawColor}{RGB}{0,0,0}

\path[draw=drawColor,line width= 1.1pt,dash pattern=on 4pt off 4pt ,line join=round] ( 33.40, 90.28) --
	( 45.34, 86.28) --
	( 57.29, 91.18) --
	( 69.23, 80.23) --
	( 81.17, 76.19) --
	( 93.12, 78.88) --
	(105.06, 73.64) --
	(117.00, 72.59) --
	(128.95, 79.48) --
	(140.89,119.05);
\definecolor{drawColor}{RGB}{102,205,170}

\path[draw=drawColor,line width= 1.1pt,dash pattern=on 2pt off 2pt on 6pt off 2pt ,line join=round] ( 33.40, 62.92) --
	( 45.34, 75.55) --
	( 57.29, 82.42) --
	( 69.23, 85.00) --
	( 81.17, 86.20) --
	( 93.12, 89.62) --
	(105.06, 90.52) --
	(117.00, 91.48) --
	(128.95, 92.79) --
	(140.89,119.05);

\path[] ( 33.40, 11.14) --
	( 45.34, 11.14) --
	( 57.29, 11.14) --
	( 69.23, 11.14) --
	( 81.17, 11.14) --
	( 93.12, 11.14) --
	(105.06, 11.14) --
	(117.00, 11.14) --
	(128.95, 11.14) --
	(140.89, 11.14);

\path[] ( 33.40, 11.14) --
	( 45.34, 11.14) --
	( 57.29, 11.14) --
	( 69.23, 11.14) --
	( 81.17, 11.14) --
	( 93.12, 11.14) --
	(105.06, 11.14) --
	(117.00, 11.14) --
	(128.95, 11.14) --
	(140.89, 11.14);

\path[] ( 33.40, 11.14) --
	( 45.34, 11.14) --
	( 57.29, 11.14) --
	( 69.23, 11.14) --
	( 81.17, 11.14) --
	( 93.12, 11.14) --
	(105.06, 11.14) --
	(117.00, 11.14) --
	(128.95, 11.14) --
	(140.89, 11.14);
\end{scope}
\begin{scope}
\path[clip] (  0.00,  0.00) rectangle (151.77,144.54);
\definecolor{drawColor}{RGB}{190,190,190}

\node[text=drawColor,rotate= 90.00,anchor=base,inner sep=0pt, outer sep=0pt, scale=  0.77] at ( 22.33, 35.12) {0.2};

\node[text=drawColor,rotate= 90.00,anchor=base,inner sep=0pt, outer sep=0pt, scale=  0.77] at ( 22.33, 59.10) {0.4};

\node[text=drawColor,rotate= 90.00,anchor=base,inner sep=0pt, outer sep=0pt, scale=  0.77] at ( 22.33, 83.08) {0.6};

\node[text=drawColor,rotate= 90.00,anchor=base,inner sep=0pt, outer sep=0pt, scale=  0.77] at ( 22.33,107.06) {0.8};
\end{scope}
\begin{scope}
\path[clip] (  0.00,  0.00) rectangle (151.77,144.54);
\definecolor{drawColor}{gray}{0.20}

\path[draw=drawColor,line width= 0.6pt,line join=round] ( 25.27, 35.12) --
	( 28.02, 35.12);

\path[draw=drawColor,line width= 0.6pt,line join=round] ( 25.27, 59.10) --
	( 28.02, 59.10);

\path[draw=drawColor,line width= 0.6pt,line join=round] ( 25.27, 83.08) --
	( 28.02, 83.08);

\path[draw=drawColor,line width= 0.6pt,line join=round] ( 25.27,107.06) --
	( 28.02,107.06);
\end{scope}
\begin{scope}
\path[clip] (  0.00,  0.00) rectangle (151.77,144.54);
\definecolor{drawColor}{gray}{0.20}

\path[draw=drawColor,line width= 0.6pt,line join=round] ( 33.40, 28.05) --
	( 33.40, 30.80);

\path[draw=drawColor,line width= 0.6pt,line join=round] ( 57.29, 28.05) --
	( 57.29, 30.80);

\path[draw=drawColor,line width= 0.6pt,line join=round] ( 81.17, 28.05) --
	( 81.17, 30.80);

\path[draw=drawColor,line width= 0.6pt,line join=round] (105.06, 28.05) --
	(105.06, 30.80);

\path[draw=drawColor,line width= 0.6pt,line join=round] (128.95, 28.05) --
	(128.95, 30.80);

\path[draw=drawColor,line width= 0.6pt,line join=round] (140.89, 28.05) --
	(140.89, 30.80);
\end{scope}
\begin{scope}
\path[clip] (  0.00,  0.00) rectangle (151.77,144.54);
\definecolor{drawColor}{RGB}{190,190,190}

\node[text=drawColor,anchor=base,inner sep=0pt, outer sep=0pt, scale=  0.78] at ( 33.40, 19.70) {.1};

\node[text=drawColor,anchor=base,inner sep=0pt, outer sep=0pt, scale=  0.78] at ( 57.29, 19.70) {.3};

\node[text=drawColor,anchor=base,inner sep=0pt, outer sep=0pt, scale=  0.78] at ( 81.17, 19.70) {.5};

\node[text=drawColor,anchor=base,inner sep=0pt, outer sep=0pt, scale=  0.78] at (105.06, 19.70) {.7};

\node[text=drawColor,anchor=base,inner sep=0pt, outer sep=0pt, scale=  0.78] at (128.95, 19.70) {.9};

\node[text=drawColor,anchor=base,inner sep=0pt, outer sep=0pt, scale=  0.78] at (140.89, 19.70) {.99};
\end{scope}
\begin{scope}
\path[clip] (  0.00,  0.00) rectangle (151.77,144.54);
\definecolor{drawColor}{RGB}{0,0,0}

\node[text=drawColor,anchor=base,inner sep=0pt, outer sep=0pt, scale=  1.21] at ( 87.15,  7.85) {\underline{\textbf{MISSING DATA}}};
\end{scope}
\begin{scope}
\path[clip] (  0.00,  0.00) rectangle (151.77,144.54);
\definecolor{drawColor}{RGB}{0,0,0}

\node[text=drawColor,anchor=base,inner sep=0pt, outer sep=0pt, scale=  0.88] at ( 87.15,132.98) {MNIST};
\end{scope}
\end{tikzpicture}} 
\end{tabular}
\caption{\textbf{Top-left corner:} Number of clusters obtained by \FSC\ as a function of the parameter $\lambdaa$ in \eqref{ifscEq}. \textbf{Rest:} Clustering error of \FSC\ and other baseline algorithms. \underline{Notice the different scales}. With full-data \FSC\ is rarely and barely outperformed by other algorithms. In contrast, when data is missing, \FSC\ outperforms other algorithms by a wide margin. For example, in simulations with $\n_\k=20$ (resp.~Yale dataset) and $\p=0.9$, \FSC\ achieves $7.5\%$ error (resp.~$25.7\%$), while the next best algorithm achieves $71.25\%$ (resp.~$64.79\%$). We point out that some curves are ``missing'' from some plots because some methods are not applicable. e.g., SCC cannot handle missing data, and Lift cannot handle large dimensions.} 
%\careful{$^\star$BDR corresponds to the (slightly better) results reported in \cite{full2}, which is the only discrepancy we were unable to reproduce.}
\label{resultsFig}
\end{figure*}

% =========LAMBDA
\subsubsection{Effect of the penalty parameter}
We first study the number of clusters obtained by \FSC\ as a function of $\lambdaa$, with the default settings above. Figure \ref{resultsFig} shows, consistent with our discussion in Section \ref{modelSelectionSec}, that if $\lambdaa=0$, \FSC\ assigns each point to its own cluster. As $\lambdaa$ increases, subspaces start fusing together up to the point where if $\lambdaa$ is too large, \FSC\ fuses all subspaces into one, and all data gets clustered together. Next we study performance. Figure \ref{resultsFig} shows that there is a wide range of values of $\lambdaa$ that produce low error, showing that \FSC\ is quite stable. Note that the error increases if $\lambdaa$ is too small or too large. This is consistent with our previous experiment, showing that extreme values of $\lambdaa$ produce too few or too many clusters.

% =========DIMENSIONALITY
\subsubsection {Effect of dimensionality} It is well-documented that data in lower-dimensional subspaces are easier to cluster \cite{infoTheoretic, hrmc, gssc, elhamifar, greg, ladmc}. In the extreme case, clustering $1$-dimensional subspaces requires a simple co-linearity test, and is theoretically possible with as little as $2$ samples per column \cite{infoTheoretic}. In contrast, no existing algorithm can successfully cluster $(\d-1)$-dimensional subspaces (hyperplanes), which is actually impossible even if one entry per column is missing \cite{infoTheoretic}. Of course, subspaces' dimensionality is relative to the ambient dimension: a $10$-dimensional subspace is a hyperplane in $\R{}^{11}$, but low-dimensional in $\R{}^{1000}$. In this experiment we test \FSC\ as a function of the {\em low-dimensionality} of the subspaces, i.e., the gap between the ambient dimension $\d$ and the subspaces' dimension $\r$. First we fix $\d=100$ and compute error as a function of $\r$. As $\r$ grows, the subspace becomes higher and higher-dimensional. Then we turn things around, fixing $\r=5$ and varying $\d$. As $\d$ grows, this subspace becomes lower and lower-dimensional. Figure \ref{resultsFig} shows that \FSC\ is more sensitive to high-dimensionality than the state-of-the-art. However, pay attention to the scale: even in the worst-case ($\r=30$), the gap between \FSC\ and the state-of-the-art is around $10\%$.

% =========NOISE
\subsubsection{Effect of noise}
\label{sec:Effect of noise} Figure \ref{resultsFig} shows that \FSC\ performs as well or better than the state-of-the-art with different noise levels. Recall that $\lambdaa$ quantifies the tradeoff between how accurately we want to represent each point $\x_\i$ (the first term in \eqref{ifscEq}), and how close subspaces from different points will be (second term), which in turn determines how subspaces fuse together, or equivalently, how many subspaces we will obtain. If data is completely noiseless, we expect to represent each point very accurately, and so we can use a smaller $\lambdaa$ (giving more weight to the first term). On the other hand, if data is noisy, we expect to represent each point within the noise level, and so we can use a larger $\lambdaa$. As a rule of thumb, we can use $\lambdaa$ inversely proportional to the noise level $\sigmaa$.

% =========K, Nk
\subsubsection{Effect of the number of subspaces and data points}
\label{sec: Effect of the number of subspaces and data points} Figure \ref{resultsFig} shows that \FSC\ is quite robust to the number of subspaces. Recall that in our default settings, $\r=5$, so $\K \geq 20$ produces a full-rank data matrix $\X$. Figure \ref{resultsFig} also evaluates the performance of \FSC\ as a function of the columns per subspace $\n_\k$. Since $\r=5$, $\n_\k=6$ is information-theoretically necessary for subspace clustering, we conclude that \FSC\ only requires little more than that to perform as well as the state-of-the-art.

% =========MISSING DATA
\subsubsection{Effect of missing data}
\label{sec:Effect of missing data} There is a tradeoff between the number of columns per subspace $\n_\k$ and the sampling rate $(1-\p)$ required for subspace clustering \cite{infoTheoretic}. The larger $\n_\k$, the higher $\p$ may be, and vice versa. Figure \ref{resultsFig} evaluates the performance of \FSC\ as a function of $\p$ with $\n_\k=20,50$ (few and many columns). Notice that if $\p \approx 0$ (few missing data), then \FSC\ performs as well as the state-of-the-art, and much better as $\p$ increases (many missing data); see for example $\n_\k=20$ and $\p=0.9$, where the best alternative algorithms gets $71.25\%$ error, which is close to random guessing (because there are $\K=4$ subspaces in our default settings). In contrast, \FSC\ gets $7.5\%$ error. Notice that $\p=0.9$ is very close to the exact information-theoretic minimum sampling rate $\p=1-(\r+1)/\d=0.94$ \cite{infoTheoretic}. Similar to noise, if there is much missing data the first term in \eqref{ifscEq} will carry less weight, which we can compensate by making $\lambdaa$ smaller.

%\careful{Say that \cite{elhamifar} has full-columns, which trivializes the problem (and explain why: because we can complete)}

%================== REAL DATA
\subsection{Real Data Experiments}
%=================== YALE
\subsubsection{ Face Clustering.} It has been shown that the vectorized images of the same person under different illuminations lie near a $9$-dimensional subspace \cite{lambertian}. In this experiment we evaluate the performance of \FSC\ at clustering faces of multiple individuals, using the Yale B dataset \cite{yale}, containing a total of $2432$ images, each of size $48 \times 42$, evenly distributed among $38$ individuals. To compare things vis \`a vis, before clustering, we use robust \PCA\ \cite{robustpca} on each cluster, to remove outliers; this is a widely used preprocessing step \cite{ssc,ewzf,ssp14,elhamifar}. In each of $30$ trials, we select $\K$ people uniformly at random, and record the clustering error. Figure \ref{resultsFig} shows that \FSC\ is very competitive and there is only a negligible gap between \FSC\ and the best alternative algorithm. Figure \ref{resultsFig} also shows the average clustering error as a function of the amount of missing data (induced uniformly at random), with $\K$ fixed to $6$ people. Consistent with our simulations, \FSC\ outperforms the state-of-the-art in the low-sampling regime (many missing data). For example, with $\p=0.9$ \FSC\ gets $25.7\%$ error, while the next best algorithm gets $64.79\%$. Note that $\p=0.9$ is quite close to the exact information-theoretic limit $\p=1-(\r+1)/\d=0.995$ \cite{infoTheoretic}.

%==========HOPKINS
\subsubsection{ Motion Segmentation.}
It is well-known that the locations over time of a rigidly moving object approximately lie in a $3$-dimensional affine subspace \cite{kanade,kanatani} (which can be thought as a $4$-dimensional subspace whose fourth component accounts for the offset). Hence, by tracking points in a video, and subspace clustering them, we can segment the multiple moving objects appearing in the video. In this experiment we test \FSC\ on this task, using the Hopkins 155 dataset \cite{hopkins}, containing sequences of points tracked over time in $155$ videos. Each video contains $\K=2,3$ objects. On average, each object is tracked on $\n_\k=133$ points (described by two coordinates) over $29$ frames, producing vectors in ambient dimension $\d=58$.
Figure \ref{resultsFig} shows the results. With full-data \FSC\ is far from the best, but has performance comparable to the rest of the algorithms. However, when data is missing, we again see that \FSC\ again dramatically outperforms the rest of the algorithms. Figure \ref{resultsFig} shows the average results over all videos when data is induced uniformly at random. For example, with $\p=0.9$, the best baseline algorithm gets $52.95\%$ error. In contrast, \FSC\ achieves $15.03\%$ error. Notice that $\p=0.9$ is very close to the exact information-theoretic minimum sampling rate $\p=10(\r+1)/\d=0.914$ \cite{infoTheoretic}.

%=================== HANDRWITTEN
\subsubsection{ Handwritten Digits Clustering.} As a last experiment we use \FSC\ to cluster vectorized images of handwritten digits, known to be well-approximated by 12-dimensional subspaces \cite{HW12}. For this purpose we use the MNIST dataset \cite{MNIST}, containing thousands of gray-scaled, $28 \times 28$ images. First we will test \FSC\ as a function of the number of digits (subspaces) in the mix. Following common practice, for each $\K=2,3,\dots,10$ we randomly selected $\K$ digits, $\n_\k=50$ images per digit, and aimed to cluster the images. Figure \ref{resultsFig} shows the average results of 10 independent trials of each configuration. Consistent with our previous experiments, if no data is missing, \FSC\ performs comparable to the rest of the algorithms with a gap (in the worst-cases) no larger than $5\%$. However, as soon as we induce missing data (uniformly at random), \FSC\ starts outperforming all other methods by a huge margin (up to $40\%$). Figure \ref{resultsFig} shows the results of this experiment.

%================== BROADER IMPACT
\section*{Broader Impact}
This paper introduces a novel strategy to address missing data in subspace clustering, which enables clustering and completion in regimes where other methods fail. Practitioners in computer vision, recommender systems, networks inference, and data science in general can use our new method. We expect this paper to motivate the learning community to explore new directions that stem from this initial work. These include the investigation of (i) ADMM and AMA formulations of \eqref{ifscEq} that reduce computational complexity (as in \cite{chi} for euclidean clustering), (ii) optimal initializations, greedy, adaptive, data-driven, and outlier-robust variants, and (iii) geodesics on the Stiefel and Grassmann manifolds (similar to \cite{grouse} for subspace tracking) to avoid the inversion of the term $\U{}_\i^\T\U_\i$ in $\P_\i$, which may become ill-conditioned. Ultimately, we hope this publication spurs discussions and insights that ultimately lead to better methods and a better understanding of subspace clustering when data is missing.

% \cite{geodesics} for Newton's method and conjugate gradient, or

%======================= REFERENCES =====================


\begin{thebibliography}{99}
% \setlength{\itemindent}{-\leftmargin}
% \makeatletter\renewcommand{\@biblabel}[1]{}\makeatother
% ======== Applications
\bibitem{charRecog}
T.~Hastie and P.~Simard, \emph{Metrics and models for handwritten character recognition}, Statistical Science, 1998.

\bibitem{kanade}
C.~Tomasi and T.~Kanade, \emph{Shape and motion from image streams under orthography}, International Journal of Computer Vision, 1992.

\bibitem{kanatani}
K.~Kanatani, \emph{Motion segmentation by subspace separation and model selection}, IEEE International Conference in Computer Vision, 2001.

\bibitem{lambertian}
R.~Basri and D.~Jacobs, \emph{Lambertian reflection and linear subspaces}, IEEE Transactions on Pattern Analysis and Machine Intelligence, 2003.

\bibitem{recommender}
J.~Rennie and N.~Srebro, \emph{Fast maximum margin matrix factorization for collaborative prediction}, International Conference on Machine Learning, 2005.

\bibitem{scc}
G.~Chen and G.~Lerman, \emph{Spectral curvature clustering (SCC)} International Journal of Computer Vision, 2009.

\bibitem{eriksson}
B.~Eriksson, P.~Barford, J.~Sommers and R.~Nowak, \emph{DomainImpute: Inferring unseen components in the Internet}, IEEE INFOCOM Mini-Conference, 2011.

\bibitem{guessWho}
A.~Zhang, N.~Fawaz, S.~Ioannidis and A.~Montanari, \emph{Guess who rated this movie: Identifying users through subspace clustering}, Uncertainty in Artificial Intelligence, 2012.

\bibitem{ssc}
E.~Elhamifar and R.~Vidal, \emph{Sparse subspace clustering: Algorithm, theory, and applications} IEEE Transactions on Pattern Analysis and Machine Intelligence, 2013.

\bibitem{network}
G.~Mateos and K.~Rajawat, \emph{Dynamic network cartography: Advances in network health monitoring}, IEEE Signal Processing Magazine, 2013.

% SUBSPACE CLUSTERING STUFF
\bibitem{liu1}
G.~Liu, Z.~Lin and Y.~Yu, \emph{Robust subspace segmentation by low-rank representation}, International Conference on Machine Learning, 2010.

\bibitem{scVidal}
R.~Vidal, \emph{Subspace clustering}, IEEE Signal Processing Magazine, 2011.

\bibitem{liu2}
G.~Liu, Z.~Lin, S.~Yan, J.~Sun, Y.~Yu and Y.~Ma, \emph{Robust recovery of subspace structures by low-rank representation}, IEEE Transactions on Pattern Analysis and Machine Intelligence, 2013.

\bibitem{mahdi}
M.~Soltanolkotabi, E.~Elhamifar and E.~Cand\`es, \emph{Robust subspace clustering}, Annals of Statistics, 2014.

\bibitem{qu}
C.~Qu and H.~Xu, \emph{Subspace clustering with irrelevant features via robust Dantzig selector}, Advances in Neural Information Processing Systems, 2015.

\bibitem{peng}
X.~Peng, Z.~Yi and H.~Tang, \emph{Robust subspace clustering via thresholding ridge regression}, AAAI Conference on Artificial Intelligence, 2015.

%Noisy
\bibitem{wang}
Y.~Wang and H.~Xu, \emph{Noisy sparse subspace clustering}, International Conference on Machine Learning, 2013.

%Privacy
\bibitem{aarti}
Y.~Wang, Y.~Wang and A.~Singh, \emph{Differentially private subspace clustering}, Advances in Neural Information Processing Systems, 2015.

%Submanifold
\bibitem{hu}
H.~Hu, J.~Feng and J.~Zhou, \emph{Exploiting unsupervised and supervised constraints for subspace clustering}, IEEE Pattern Analysis and Machine Intelligence, 2015.

%============Additional after NIPS reviews
\bibitem{scalableSC}
C.~You, D.~Robinson and R.~Vidal, \emph{Scalable sparse subspace clustering by orthogonal matching pursuit}, IEEE Conference on Computer Vision and Pattern Recognitionm 2016.

\bibitem{L0sparse}
Y.~Yang, J.~Feng, N.~Jojic, J.~Yang and T.~Huang, \emph{$\ell_0$-sparse subspace clustering}, European Conference on Computer Vision, 2016.

\bibitem{dataDependent}
B.~Xin, Y.~Wang, W.~Gao and D.~Wipf, \emph{Data-dependent sparsity for subspace clustering}, Uncertainty in Artificial Intelligence, 2017.

%========Full data algorithms
\bibitem{LRSC}
P.~Favaro, R.~Vidal, and A.~Ravichandran, \emph{A closed form solution to robust subspace estimation and clustering}, IEEE Conference on Computer Vision and Pattern Recognition, 2011.

\bibitem{LSR}
C.~Lu, H.~Min, Z.~Zhao, L.~Zhu, D.~Huang, and S.~Yan, \emph{Robust and efficient subspace segmentation via least squares regression}, Proceedings of the  European Conference in Computer Vision, 2012.

\bibitem{greedySC}
D.~Park, C.~Caramanis and S.~Sanghavi, \emph{Greedy subspace clustering}, Advances in Neural Information Processing Systems, 2014.

\bibitem{full1}
X.~Peng, Z.~Yu, Z.~Yi and H. Tang, \emph{Constructing the L2-graph for robust subspace learning and subspace clustering}, IEEE Transactions on Cybernetics, 2017.

\bibitem{full3}
P.~Ji, T.~Zhang, H.~Li, M.~Salzmann and I.~Reid, \emph{Deep subspace clustering networks}, Advances in Neural Information Processing Systems, 2017.

\bibitem{full4}
M.~Rahmani and G.~Atia, \emph{Innovation pursuit: a new approach to subspace clustering}, IEEE Transactions on Signal Processing, 2017.

\bibitem{full5}
C.~Li, C.~You and R.~Vidal, \emph{Structured sparse subspace clustering: a joint affinity learning and subspace clustering framework}, IEEE Transactions on Image Processing, 2017.

\bibitem{full2}
C.~Lu, J.~Feng, Z.~Lin, T.~Mei and S.~Yan, \emph{Subspace clustering by block diagonal representation}, IEEE Transactions on Pattern Analysis and Machine Intelligence, 2018.

\bibitem{full6}
M.~Yin, S.~Xie, Z.~Wu, Y.~Zhang and J.~Gao, \emph{Subspace clustering via learning an adaptive low-rank graph}, IEEE Transactions on Image Processing, 2018.


%========INPAINTING
\bibitem{inpainting}
J.~Mairal, F.~Bach, J.~Ponce and G.~Sapiro, \emph{Online dictionary learning for sparse coding}, International Conference on Machine Learning, 2009.

\bibitem{occlusions}
R.~Vidal, R.~Tron and R.~Hartley, \emph{Multiframe motion segmentation with missing data using Power Factorization and GPCA} International Journal of Computer, 2008.

\bibitem{collaborativeRanking}
D.~Park, J.~Neeman, J.~Zhang, S.~Sanghavi and I.~Dhillon, \emph{Preference completion: Large-scale collaborative ranking from pairwise comparisons}, International Conference on Machine Learning, 2015.


%========iSC
\bibitem{ewzf}
C.~Yang, D.~Robinson and R.~Vidal, \emph{Sparse subspace clustering with missing entries}, International Conference on Machine Learning, 2015.

\bibitem{candes-recht}
E.~Cand\`es and B.~Recht, \emph{Exact matrix completion via convex optimization}, Foundations of Computational Mathematics, 2009.

\bibitem{tsakiris}
M.~Tsakiris and R.~Vidal, \emph{Theoretical analysis of sparse subspace clustering with missing entries}, International Conference on Machine Learning, 2018.

\bibitem{elhamifar}
E.~Elhamifar, \emph{High-rank matrix completion and clustering under self-expressive models}, Neural Information Processing Systems, 2016.

\bibitem{hrmc}
B.~Eriksson, L.~Balzano and R.~Nowak, \emph{High-rank matrix completion and subspace clustering with missing data}, Artificial Intelligence and Statistics, 2012.

\bibitem{kGROUSE}
L.~Balzano, R.~Nowak, A.~Szlam and B.~Recht, \emph{k-Subspaces with missing data}, IEEE Statistical Signal Processing, 2012.

\bibitem{ssp14}
D.~Pimentel-Alarc\'on, L.~Balzano and R.~Nowak, \emph{On the sample complexity of subspace clustering with missing data}, IEEE Statistical Signal Processing, 2014.

\bibitem{gssc}
D.~Pimentel-Alarc\'on, L.~Balzano, R.~Marcia, R.~Nowak and R.~Willett, \emph{Group-sparse subspace clustering with missing data}, IEEE Statistical Signal Processing, 2016.

%=====Liftings
\bibitem{greg}
G.~Ongie, R.~Willett, R.~Nowak and L.~Balzano, \emph{Algebraic variety models for high-rank matrix completion}, International Conference on Machine Learning, 2017.

\bibitem{ladmc}
D.~Pimentel-Alarc\'on, G.~Ongie, L.~Balzano, R.~Willett and R.~Nowak, \emph{Low algebraic dimension matrix completion}, Allerton Conference on Communication, Control, and Computing, 2017.

%========Fuzzy - papers suggested by reviewers
\bibitem{latentFactorAnalysis}
Y. Song, M. Li, Z. Zhu, G. Yang and X. Luo, \emph{Non-negative Latent Factor Analysis-Incorporated and Feature-Weighted Fuzzy Double c-Means Clustering for Incomplete Data}," IEEE Transactions on Fuzzy Systems.

\bibitem{missingValueImputation}
D. Li, H. Zhang, T. Li, A. Bouras, X. Yu and T. Wang, \emph{Hybrid Missing Value Imputation Algorithms Using Fuzzy C-Means and Vaguely Quantified Rough Set}, IEEE Transactions on Fuzzy Systems.

\bibitem{KNNProblem}
Beretta, L., Santaniello, A. Nearest neighbor imputation algorithms: a critical evaluation. BMC Med Inform Decis Mak 16, 74 (2016).

\bibitem{RegressionProblem1}
Lodder P. To Impute or not Impute, That’s the Question. In Mellenbergh GJ, Adèr HJ, editors, Advising on research methods: Selected topics 2013. Huizen: Johannes van Kessel Publishing. 2014.

\bibitem{RegressionProblem2}
A. Chaudhry et al, "A Method for Improving Imputation and Prediction Accuracy of Highly Seasonal Univariate Data with Large Periods of Missingness," Wireless Communications \& Mobile Computing (Online), vol. 2019, pp. 13, 2019.

\bibitem{CompareApproaches}
C. Lane, R. Boger, C. You, M. Tsakiris, B. Haeffele and R. Vidal, "Classifying and Comparing Approaches to Subspace Clustering with Missing Data," 2019 IEEE/CVF International Conference on Computer Vision Workshop (ICCVW), 2019, pp. 669-677.

%======Gaussian Incoherence
\bibitem{gaussianIncoherence}
E.~Cand\`es, Y.~Eldar, D.~Needell and P.~Randall, \emph{Compressed sensing with coherent and redundant dictionaries}, Applied and Computational Harmonic Analysis, 2011.

%======Grassmannian Distance
\bibitem{grassDist}
K.~Ye and L.~Lim, \emph{Schubert varieties and distances between subspaces of different dimensions}, SIAM Journal on Matrix Analysis and Applications, 2016.

%======Group Lasso
\bibitem{antoniadisFan}
A.~Antoniadis and J.~Fan, \emph{Regularization of wavelet approximations (with discussion)}, Journal of the American Statistical Association, 2001.

\bibitem{yuanLin}
M.~Yuan and Y.~Lin, \emph{Model selection and estimation in regression with grouped variables}, Journal of the Royal Statistical Society, 2006.

\bibitem{meier}
L.~Meier, S.~Van de Geer and P.~B\"{u}hlmann, \emph{The group lasso for logistic regression}, Journal of the Royal Statistical Society, 2008.

\bibitem{tibshirani}
J.~Friedman, T.~Hastie and R.~Tibshirani, \emph{A note on the group lasso and a sparse group lasso}, Arxiv preprint, 2010.

\bibitem{pkdd}
M.~Kshirsagar, E.~Yang and A.~Lozano, \emph{Learning task structure via sparsity grouped multitask learning}, European Conference on Machine Learning and Principles and Practice of Knowledge Discovery in Databases, 2017.

\bibitem{msdmd}
L.~Balzano, B.~Recht and R.~Nowak, \emph{High-dimensional matched subspace detection when data are missing}, IEEE International Symposium on Information Theory, 2010.

%======Fusion penalties
\bibitem{fusionVariable}
S.~Land and J.~Friedman, \emph{Variable fusion: a new method of adaptive signal regression}, Technical Report, Department of Statistics, Stanford University, 1996.

\bibitem{fusedLasso}
R.~Tibshirani, S.~Rosset, J.~Zhu and K.~Knight, \emph{Sparsity and smoothness via the fused lasso}, Journal of the Royal Statistical Society, 2005.

\bibitem{groupPursuit}
X.~Shen and H.~Huang, \emph{Grouping pursuit through a regularization solution surface}, Journal of the American Statistical Association, 2010.

\bibitem{clusterpath}
T.~Hocking, A.~Joulin and F.~Bach, \emph{Clusterpath: An algorithm for clustering using convex fusion penalties} International Conference on Machine Learning, 2011.

\bibitem{sumofnorms}
F.~Lindsten, H.~Ohlsson and L.~Ljung, \emph{Clustering using sum-of-norms regularization: with application to particle filter output computation}, Statistical Signal Processing, 2011.

\bibitem{fusion}
S.~Poddar and M.~Jacob, \emph{Clustering of data with missing entries using non-convex fusion penalties}, arXiv preprint, 2017.

\bibitem{infoTheoretic}
D.~Pimentel-Alarc\'on and R.~Nowak, \emph{The information-theoretic requirements of subspace clustering with missing data}, International Conference on Machine Learning, 2016.

\bibitem{powell}
M.~Powell, \emph{On search directions for minimization algorithms}, Mathematical Programming, 1973.

\bibitem{nocedal}
J.~Nocedal and S.~Wright, \emph{Numerical optimization}, Springer Science and Business Media, 2006.

\bibitem{beck}
A.~Beck, \emph{On the convergence of alternating minimization for convex programming with applications to iteratively reweighted least squares and decomposition schemes}, SIAM Journal on Optimization, 2015.

\bibitem{boyd}
S.~Boyd and L.~Vandenberghe, \emph{Convex optimization}, Cambridge University Press, 2004.

\bibitem{ding}
C.~Ding, D.~Zhou, X.~He and H.~Zha, \emph{R1-PCA: Rotational invariant L1-norm principal component analysis for robust subspace factorization}, International Conference on Machine Learning, 2006.

\bibitem{spectral}
A.~Ng, M.~Jordan and Y.~Weiss, \emph{On spectral clustering: analysis and an algorithm}, Advances in Neural Information Processing Systems, 2002.

\bibitem{akaike}
H.~Akaike, \emph{Information theory and an extension of the maximum likelihood principle}, IEEE International Symposium on Information Theory, 1973.

\bibitem{chi}
E.~Chi and K.~Lange, \emph{Splitting methods for convex clustering}, Journal of Computational and Graphical Statistics, 2015.

\bibitem{tan}
K.~Tan and D.~Witten, \emph{Statistical properties of convex clustering}, Electronic Journal of Statistics, 2015.

%%=======ADMM
%\bibitem{admm1}
%D.~Gabay and B.~Mercier, \emph{A dual algorithm for the solution of nonlinear variational problems via finite-element approximations}, Computers \& Mathematics with Applications, 1976.
%
%\bibitem{admm2}
%S.~Boyd, N.~Parikh, E.~Chu, B.~Peleato and J.~Eckstein, \emph{Distributed optimization and statistical learning via the alternating direction method of multipliers}, Foundations and Trends in Machine Learning, 2010.

%===== Our code
%\bibitem{ourCode}
%Hidden for anonymity purposes; attached with submission.

% ==== Yale
\bibitem{yale}
A.~Georghiades, P.~Belhumeur and D.~Kriegman, \emph{From few to many: Illumination cone models for face recognition under variable lighting and pose}, IEEE Transactions on Pattern Analysis and Machine Intelligence, 2001.

% Robust PCA
\bibitem{robustpca}
E.~Cand\`es, X.~Li, Y.~Ma and J.~Wright, \emph{Robust principal component analysis?}, Journal of the ACM, 2011.

%======== hopkins
\bibitem{hopkins}
R.~Tron and R.~Vidal, \emph{A benchmark for the comparison of 3-D motion segmentation algorithms}, IEEE Conference on Computer Vision and Pattern Recognition, 2007.

\bibitem{grouse}
L.~Balzano, R.~Nowak and B.~Recht, \emph{Online identification and tracking of subspaces from highly incomplete information}, Allerton Conference on Communication, Control and Computing, 2010.

% HandWritten Digits
\bibitem{HW12}
T.~Hastie and P.~Simard, \emph{Metrics and models for handwritten character recognition}, Statistical Science, 1998.

\bibitem{MNIST}
Y.~LeCun, L.~Bottou, Y.~Bengio and P.~Haffner, \emph{Gradient-based learning applied to document recognition}, Proceedings of the IEEE, 1998.

% \bibitem{geodesics}
% A.~Edelman, T.~Arias, and S.~Smith, \emph{The geometry of algorithms with orthogonality constraints}, SIAM journal on Matrix Analysis and Applications, 1998.

\end{thebibliography}
\end{document}